\documentclass{article}
\pdfoutput=1
\usepackage[utf8x]{inputenc}
\usepackage[T1]{fontenc}

\usepackage[accepted]{icml2021_style/icml2021}

\usepackage{amsfonts, amsmath, amssymb, amsthm}
\usepackage{xfrac}
\usepackage{graphicx, xcolor, colortbl}
\usepackage{footnote}
\usepackage{tablefootnote}
\usepackage{natbib}
\usepackage{dsfont}
\usepackage{bbm}
\usepackage{import}
\usepackage{mathtools}
\usepackage{bm}
\usepackage{xspace}
\usepackage{thmtools,thm-restate} 
\usepackage{accents}
\usepackage{framed}
\usepackage[inline]{enumitem}
\usepackage{booktabs}
\usepackage{color, colortbl}
\definecolor{LightGray}{gray}{0.9}

\usepackage{scrextend}

\usepackage{multirow}

\newtheorem{lemma}{Lemma}
\newtheorem{theorem}{Theorem}

\theoremstyle{definition}

\usepackage{sty/notations}

\usepackage{todonotes}

\usepackage{minitoc}

\icmltitlerunning{UCB Momentum Q-learning: Correcting the bias without forgetting}

\begin{document}

\twocolumn[
\icmltitle{UCB Momentum Q-learning:\\
Correcting the bias without forgetting}



\icmlsetsymbol{equal}{*}

\begin{icmlauthorlist}
\icmlauthor{Pierre M\'enard}{ovgu}
\icmlauthor{Omar Darwiche Domingues}{inria}
\icmlauthor{Xuedong Shang}{inria,lille}
\icmlauthor{Michal Valko}{inria,lille,deepmind}
\end{icmlauthorlist}

\icmlaffiliation{inria}{Inria}
\icmlaffiliation{lille}{Universit\'e de Lille}
\icmlaffiliation{deepmind}{DeepMind Paris}
\icmlaffiliation{ovgu}{Otto von Guericke University}

\icmlcorrespondingauthor{Pierre M\'enard}{pierre.menard@ovgu.de}
\icmlcorrespondingauthor{Omar Darwiche Domingues}{omar.darwiche-domingues@inria.fr}

\icmlkeywords{Machine Learning, ICML}

\vskip 0.3in
]



\printAffiliationsAndNotice{}  


\doparttoc 
\faketableofcontents 

\begin{abstract}\noindent
We propose \OurAlgorithm, Upper Confidence Bound Momentum Q-learning, a new algorithm for  reinforcement learning in tabular and possibly stage-dependent, episodic Markov decision process. \OurAlgorithm is based on Q-learning where we add a momentum term and rely on the principle of optimism in face of uncertainty to deal with exploration.
Our new technical ingredient of \OurAlgorithm is the use of momentum to correct the bias that Q-learning suffers while, \emph{at the same time}, limiting the impact it has on the the second-order term of the regret. For \OurAlgorithm, we are able to guarantee a regret of at most $\tcO(\sqrt{H^3SAT}+ H^4 S A )$ where $H$ is the length of an episode, $S$ the number of states, $A$ the number of actions, $T$ the number of episodes and ignoring terms in $\operatorname{poly\,log}(SAHT)$.
Notably, \OurAlgorithm is the first algorithm that simultaneously matches the lower bound of $\Omega(\sqrt{H^3SAT})$ for large enough $T$ and has a second-order term (with respect to the horizon $T$) that scales \emph{only linearly} with the number of states $S$.
\end{abstract}

\section{Introduction}
\label{sec:intro}

\begin{table*}[t!]
    \centering
	\label{tab:upper_bounds}
	\begin{tabular}{@{}ll@{}}
		\toprule
		\bfseries{Algorithm} & {\bf Upper bound} (non-stationary case) \\
		\midrule
		\midrule
    \UCBVI~\citep{azar2017minimax} & $\tcO(\sqrt{H^3 SA T} + H^3S^2A)$ \\
    \UBEV~\cite{dann2017unifying} & $\tcO(\sqrt{H^4 SA T} + H^2 S^3 A^2)$ \\
     \EULER~\citep{zanette2019euler} & $\tcO\Big(\sqrt{H^3 SAT} + H^3 S^{3/2}A(\sqrt{S}+\sqrt{H})\Big)$ \\
		 \OptQL~\citep{jin2018is} & $\tcO(\sqrt{H^4 SAT} + H^{9/2} S^{3/2}A^{3/2})$ \\
		 \UCBAdventage~\citep{zhang2020advantage} &
		 $\tcO(\sqrt{H^3 SAT} + H^{33/4} S^2 A^{3/2}T^{1/4})$ \\
		 \midrule
		 \rowcolor{LightGray}
		 \OurAlgorithm~(this paper) & $\tcO(\sqrt{H^3SAT}+ H^4 \textcolor{red}{S} A )$ \\
		\bottomrule
	\end{tabular}
\caption{Regret upper bound under unknown episodic, non-stationary, tabular MDPs.}
\end{table*}

In reinforcement learning (RL), an agent interacts with an environment with the objective of maximizing the sum of collected rewards \citep{sutton1998}. We model the environment as an unknown episodic tabular Markov Decision Process (MDP) with $S$ states, $A$ actions and episodes of length $H$. After $T$ episodes, we measure the performance of the agent by its cumulative regret which is the difference between the total reward collected by an optimal policy and the total reward collected by the agent during the learning. In order to minimize the regret the agent needs to balance the exploration of the environment and exploitation of the current knowledge to act optimally.

In particular, we study the \emph{non-stationary} setting where rewards and transitions can change within an episode, and for which \citet{jin2018is} and \citet{domingues2020lb} provide a problem-independent lower bound on the regret of order $\Omega(\sqrt{H^3SAT})$ (see also \citealt{azar2017minimax} for stationary transitions).

Following the previous work on the infinite-horizon setting \citep{jaksch2010ucrl,fruit2018efficient, talebi2018variance}, a first line of research on episodic MDPs~\citep{azar2017minimax,dann2017unifying,zanette2019euler} investigate model-based algorithms. The idea is to perform an optimistic value-iteration with an estimated model (i.e. estimated transitions here), and act greedily with respect to the obtained upper bounds on the optimal Q-values.

In particular, \citet{azar2017minimax} provide an upper bound on the regret of order $\tcO(\sqrt{H^3 S A T} + H^3S^2A)$.
This bound matches the lower bound for $T\geq H^3 S^3 A$, where the first-order term, $\sqrt{H^3 S A T},$ dominates.
However, for $T\leq H^3 S^3 A$, which is an important regime, the bound is affected by the second order term that scales in $S^2$ and can be harmful.
Indeed, when the number of states is very large (e.g., for continuous states MDPs after discretization), the second order term can dominate the regret bound, which in such case leads to a bound with a potentially sub-optimal rate (see~\citealt{domingues2020regret,sinclair2020adaptive}). Furthermore, in order to obtain a non-trivial upper bound on the regret (i.e., a bound smaller than $HT$), at least $H^3S^2A$ samples are needed. That means we roughly need $H^2S$ samples per state-action pair while we rather expect to have a meaningful bound with only $\text{poly}(H)$ samples per state-action pair. In the current analyses, the $S^2$ factor in the second-order term comes from the fact that, for model-based algorithms, the estimated transitions and the upper confidence bounds on the optimal value functions are correlated.\footnote{It is the same reason why there is an extra factor $S$ in the first order term of the bound of \UCRL algorithm by~\citet{jaksch2010ucrl}. This factor is "pushed" to the second-order term by the improved analysis of \citet{azar2017minimax}.} A union bound over a covering of all possible value functions with a cardinal that scales exponentially with the number of states $S$ is (implicitly) used to break the correlation. A similar remark also holds for other model-based algorithms like \EULER (see Table~\ref{tab:upper_bounds} for details).

A second line of work initiated by~\citet{jin2018is} consider model-free algorithms based on Q-learning~\citep{watkins1992q}. Interestingly, such an approach does not suffer from the same issue as model-based algorithms. Indeed, the Q-values are estimated in an online fashion (see Section~\ref{sec:algorithm_intuition}), and there is no correlation issue anymore as for model-based algorithms. On the other hand, the current estimate of the optimal Q-value for Q-learning-based algorithms relies on the target computed with past estimates of the same quantity (possibly inaccurate), therefore they suffer from a larger bias (see Section~\ref{sec:algorithm_intuition}).

In particular, \citet{jin2018is} propose to use a more aggressive learning rate to mitigate that bias by forgetting old estimates, but at the price of increasing the variance. It leads to a regret bound of order $\tcO(\sqrt{H^4SAT})$ with an extra $\sqrt{H}$ in the first-order term with respect to the lower bound.\footnote{Specifically, with Hoeffding-type bonuses they have an extra $H$ and second-order term of order $H^2SA$; with Bernstein type bonuses, the discrepancy is only of a factor $\sqrt{H}$, but the second-order term is no longer linear in $S$,
see Table~\ref{tab:upper_bounds}.} Building on variance reduction techniques, \citet{sidford2018variance} and \citet{zhang2020advantage} manage to avoid this extra dependency on the horizon. The idea is to first provide a rough estimate of the optimal value, namely the value reference function, and then leverage the low variance of a reference-advantage decomposition of the optimal Q-value to compensate the forgotten past samples. However, in their current analyses, the initial phase of learning the reference value functions degrades the second order term and brings back a $S^2$ factor (see Table~\ref{tab:upper_bounds}).

In this paper we rather follow another approach. Following the work of~\citet{azar2011speedy} (see also \citealt{weng2020momentum}), we propose \OurAlgorithm, which adds a momentum term to the targets in the Q-value updates so as to correct the bias of Q-learning. However, contrary to the generative setting considered by~\citet{azar2011speedy} where all state-action pairs are sampled at each update of the  Q-value, we have to deal with two additional challenges in our setting. First, we need to handle the exploration and we do it by introducing optimism. Second, in the absence of the oracle we do not see all state-action pairs at each "episode", but \emph{only the ones encountered along the trajectory}.
Consequently, each state-action pair learns at its own pace.

To address the above two challenges, we build a \emph{value function for each state-action pair} that represents the bias of this particular pair, and use it to build a momentum term that is able to correct the bias of previous estimates on the Q-value. Every new sample is thus used to refine the estimate on the Q-value via the target and correct the bias of the past targets via the momentum term at the same time. Moreover, with the \emph{careful} use of a Freedman-Bernstein-type inequality we manage to obtain tight dependence on the horizon without degrading the second-order term.

Using the above techniques, we prove a regret bound of order $\tcO(\sqrt{H^3SAT}+H^4SA)$ for \OurAlgorithm. This upper bound matches the lower bound up to $\operatorname{poly\,log}$ factors in $S,A,H,T$ for $T\geq H^5 S A$. This rate improves over the one of previous model-free algorithms and, for $S\geq H$, the one of previous model-based algorithms. In particular, we provide an algorithm that enjoys a second-order term \emph{only} in $S$ instead of $S^2$. Our results make a step towards resolving an open question
that was hinted by ~\citet{azar2011speedy}
and also recently explicitly raised by~\citet{zhang2020rl}. Finally, in Section~\ref{sec:experiments}, we provide numerical simulations on a grid-world environment to illustrate the benefits of not forgetting the targets in \OurAlgorithm.

We highlight our main contributions:
  \begin{itemize}
    \item We carefully  design a momentum term Q-learning in the episodic setting
    and analyze its  benefits for the regret guarantees.
   \item We propose \OurAlgorithm, with a regret bound of order $\tcO(\sqrt{H^3SAT}+H^4SA)$. It is the first algorithm, up to our knowledge, that matches the problem-independent lower bound $\Omega(\sqrt{H^3SAT})$ up to $\operatorname{poly\,log}$ terms and has a second-order term that is linear in $S$.

\end{itemize}


\section{Setting}
\label{sec:setting}

In this paper, we consider a tabular episodic MDP $\left(\cS, \cA, H, \{p_h\}_{h\in[H]},\{r_h\}_{h\in[H]}\right)$, with $\cS$ the set of states, $\cA$ the set of actions, $H$ the number of steps in one episode, $p_h(s'|s,a)$ is the probability transition from state~$s$ to state~$s'$ by taking the action $a$ at step $h,$ and $r_h(s,a)\in[0,1]$ is the bounded deterministic reward received after taking the action $a$ in state $s$ at step $h$. Note that we consider the general case of rewards and transition functions that are possibly non-stationary, i.e., that may change over the decision steps $h\in[H]$\footnote{For any integer $n\in\N^\star$, we define $[n]\eqdef\{1,\ldots,n\}$.} within an episode. We denote by $S$ and $A$ the number of states and actions, respectively.

\paragraph{Policy \& value functions.} A \emph{deterministic} policy $\pi$ is a collection of functions $\pi_h : \cS \rightarrow \cA$ for all $h\in [H]$, where every $\pi_h$  maps each state to a \emph{single} action. The value functions of $\pi$, denoted by $V_h^\pi$, as well as the optimal value functions, denoted by $\Vstar_h$ are given respectively by the Bellman equations \citep{puterman1994}:
{\small
\begin{align*}
    Q_h^{\pi}(s,a) &= r_h(s,a) + p_h V_{h+1}^\pi(s,a) &   V_h^\pi(s) &= \pi_h Q_h^\pi (s)\,.
\end{align*}
}%
By convention, $V_{H+1}^\pi \triangleq 0$. Furthermore, $p_{h} f(s, a) \triangleq \E_{s' \sim p_h(\cdot | s, a)} \left[f(s')\right]$   denotes the expectation operator with respect to the transition probabilities $p_h$ and
$(\pi_h g)(s)  \triangleq \pi_h g(s) \triangleq  g(s,\pi_h(s))$ denotes the composition with the policy~$\pi$ at step $h$. An optimal policy $\pistar$ is such that $\pistar \in \argmax_{\pi} V_1^\pi(s_1)$. The optimal Q-value and value functions are the ones of an optimal policy. Precisely we have $\Vstar_h = V_h^{\pistar}$ and $\Qstar_h =  Q_h^{\pistar}$ for all $h$.

\paragraph{Learning problem.} The agent, to which the transitions are \emph{unknown}, interacts with the environment during $T$ episodes of length $H$, with a \emph{fixed} initial state $s_1$.\footnote{As explained by \citet{fiechter1994efficient} and \citet{kaufmann2021rf}, if the first state is sampled randomly as $s_1\sim p_0,$ we can simply add an artificial first state $s_0$
such that for  any action $a$, the transition probability is defined as the distribution $p_0(s_0,a) \triangleq p_0.$} Before each episode $t$ the agent selects a policy $\pi^t$ based only on the past observed transitions up to episode $t-1$. At each step $h\in[H]$ of episode $t$, the agent observes a state $s_h^t\in\cS$, takes an action $\pi_h^t(s_h^t) = a_h^t\in\cA$ and  makes a transition to a new state $s_{h+1}^t$ according to the probability distribution $p_h(s_h^t,a_h^t)$ and receives a deterministic reward $r_h(s_h^t,a_h^t)$.

\paragraph{Regret.} We measure the agent performance through regret, which is the difference between what it could obtain (in expectation) by acting optimally and what it really gets,
\[
R^T \triangleq \sum_{t=1}^T \Vstar_1(s_1)- V_1^{\pi^t}(s_1)\,.
\]

\paragraph{Notation.} We denote the number of visits of state-action pair $(s,a)$ by $n_h^t(s,a) = \sum_{k = 1}^t \chi_h^t(s,a)$ where $\chi_h^t(s,a)$ is the indicator function $\chi_h^t(s,a) \triangleq \ind_{\{(s_h^t,a_h^t) = (s,a)\}}$. We also use the indicator function $\chi_h^t(s) \triangleq \ind_{\{s_h^t= s\}}$ to represent the event where state $s$ is visited at step $h$ in episode $t$. We denote by $p_h^t$ the Dirac distribution at $(s_{h+1}^t)$, i.e., for all functions $f$ defined on $\cS$ we have $(p_h^t f)(s,a) = f(s_{h+1}^t)$. In particular, this distribution does not depend on $(s,a)$.

\section{\texorpdfstring{\OurAlgorithm}{} algorithm}
\label{sec:algorithm}
Before presenting the algorithm we provide an intuition of how it works.
\subsection{Intuition}\label{sec:algorithm_intuition}
If the agent knows the transition probabilities, it could perform real-time Q-value iteration and obtain a bounded regret (see~\citealt{efroni2019tight}).
In this case upper bounds on the Q-value functions are updated as follows\footnote{We index the quantities by $n$ in this section where $n$ is the number of times the state-action pair $(s,a)$ is visited. In particular this is different from the time $t$ since, in our setting, all the state-action pair are not visited at each episode. See Section~\ref{sec:algorithm_description} for precise notation.}
\begin{align}
    \uQ_h^n(s,a) &= (r_h + p_h \uV_h^{n-1})(s,a) \label{eq:real_time_Q_iteration}\,,
\end{align}
where upper bounds on the optimal value functions are defined by $\uV_h^{n}(s) = \max_a \uQ_h^n(s,a)$ and initialized to $\uV_h^{0}(s) = H $. When the model is unknown we can approximate it by averaging successive sample updates as in Q-learning~\citep{watkins1992q},
{\small
\begin{align}
    Q_h^n(s,a) &= \alpha_n(r_h + p_h^n \uV_h^{n-1})(s,a)+ \big(1-\alpha_n\big)Q_h^{n-1}(s,a)\,.
\end{align}
}%
A usual choice for the learning rate is $\alpha_n = 1/n$ instead of $\alpha_n=1$ used for real-time Q-value iteration above. Unfolding the previous inequality and using Azuma–Hoeffding inequality to move for the sample expectation $p_h^i$ to the true expectation $p_h$, we have with high probability
\begin{align}
    Q_h^n(s,a) &\approx r_h(s,a) + \frac{1}{n} \sum_{i=1}^{n} p_h^i \uV_{h+1}^{i-1}(s,a)\nonumber\\
    &\approx r_h(s,a)+ p_h\!\!\!\!\!\!\!\underbrace{\left(\frac{1}{n} \sum_{i=1}^{n} \uV_{h+1}^{i-1}\right)}_{:=V_{h,s,a}^n\text{ bias-value function}}\!\!\!\!\!\!(s,a)\pm\!\!\!\!\! \underbrace{\sqrt{ \frac{H^2}{n}}}_{\text{variance term}} \label{eq:1_over_N_sum}\!\!,
\end{align}
where the bias-value function of state-action $(s,a)$ encodes the bias of the estimate $Q_h^n$ with respect to the randomness of the $(p_h^i)_{i\geq 1}$. Thus choosing (Hoeffding-type) bonuses of order $\beta^n(s,a) \approx \sqrt{H^2/n}$, we can build upper bounds on the optimal Q-value and the value functions
{\small
\begin{align*}
  \uQ_h^n(s,a) = Q_h^n(s,a) +\beta_h^n(s,a), \quad
  \uV_h^n(s) = \max_{a\in\cA}\uQ_h^n(s,a)\,.
\end{align*}
}%
However, the bias term in~\eqref{eq:1_over_N_sum} is too large because of the old (and potentially inaccurate) upper bound $\uV_{h+1}^i$ that appears in the bias-value function $V_{h,s,a}^n$. Indeed it is not clear how to prove a bound that is not exponential in the horizon $H$ in this case (see~\citealt{jin2018is}). Note that on contrary when the model is known, i.e. using \eqref{eq:real_time_Q_iteration}, we have a
smaller $V_{h,s,a}^n = \uV_{h+1}^{n-1}$ bias provided that the $(\uV_{h+1}^{i})_{i\geq 1}$
 are non-increasing.

To overcome this issue, \citet{jin2018is} propose with the \OptQL algorithm\footnote{ With Hoeffding-type bonuses.} to choose a learning rate of order $\alpha_n \approx H/n$ to keep only the recent upper-bounds $\uV_{h+1}^i$ in the bias-value value function. Indeed, proceeding as above, we have
\begin{align}
    &Q_h^n(s,a) \approx r_h(s,a)+ \frac{H}{n}\sum_{i \geq  n-H/n}^n p_h^i \uV_{h+1}^{i-1}(s,a)\nonumber\\
    &\approx r_h(s,a)+ p_h\!\!\! \underbrace{\left(\!\!\frac{H}{n}\!\!\sum_{i\geq n-n/H}^n \uV_{h+1}^{i-1}\right)}_{:=V_{h,s,a}^n\text{bias-value function}}\!\!(s,a) \pm \!\!\!\!\!\underbrace{\sqrt{\frac{H^3}{n}}}_{\text{variance term}}\!\!\!.
\label{eq:H_over_N_sum}
\end{align}
Because of the aggressive learning rate of order $H/n$ there are only $n/H$ samples in the sum of \eqref{eq:H_over_N_sum} leading to a high variance. Thus we need to add an extra $H$ factor in the bonus which leads to the sub-optimal regret bound of order $\widetilde{O}(\sqrt{H^5 S A T})$.
One workaround for this issue is to learn a reference value function~\citep{zhang2020advantage}, but it is not clear how to obtain a second order term that depends linearly on the size of the state space with this approach.

We consider another approach in this paper. Following the work by~\citet{azar2011speedy}, we add a momentum term in the update of the Q-value that corrects the bias at the price of a small vanishing increase of the variance. Precisely for a momentum rate $\gamma_n$, we now consider the following update,
{\small
\begin{align*}
    Q_h^n(s,a) &= \alpha_n (r_h + p_h^n \uV_{h+1}^{n-1})(s,a)+ (1-\alpha_n)Q_h^{n-1}(s,a)\\
    &\quad+ \gamma_n p_h^n(\uV_{h+1}^{n-1}-V_{h,s,a}^{n-1})(s,a)\,,
\end{align*}
}%
where we call $V_{h,s,a}^{n-1}$ the bias-value function of state-action $(s,a)$ defined by
{\small
\begin{align*}
    V_{h,s,a}^n(s') &= (\alpha_n+\gamma_n) \uV_{h+1}^{n-1}(s')+(1-\alpha_n-\gamma_n)V_{h,s,a}^{n-1}(s')\,.
\end{align*}
}%
Note that there is a priori a different bias-value function for each state-action pair. In particular if we force the sequence of upper bounds on the value functions to be non-increasing, it holds that $V_{h,s,a}^n-\uV_{h+1}^n \geq 0$.
We choose $\alpha_n \approx 1/n$ to not forget samples as in~\eqref{eq:H_over_N_sum}. The momentum rate is $\gamma_n\approx H/n$ to correct the bias that will appear otherwise as in~\eqref{eq:1_over_N_sum}.
As explained by \citet{azar2011speedy}, this aggressive momentum will be compensated by the fact that $\uV_{h+1}^{n-1}-V_{h,s,a}^{n-1}$ is small when the two quantities converge toward $\Vstar_{h+1}$. Thanks to these choices, the bias-value function is the same as in \eqref{eq:H_over_N_sum},
\begin{align*}
   V_{h,s,a}^n(s') &\approx \frac{H+1}{n}(V_{h,s,a}^{n-1}-\uV_h^{n-1})(s') + \uV_h^{n-1}(s') \\
   &\approx \frac{H}{n} \sum_{i\geq n-n/H}^n  \uV_{h+1}^{i-1}(s')\,.
\end{align*}
Now we explain why $ V_{h,s,a}^n$ is named \emph{bias-value function}. We have, with high probability,
{\small
\begin{align*}
    Q_h^n(s,a) &\approx r_h(s,a) \!+\!\frac{1}{n}\sum_{i=1}^n p_h^i\left( (H+1)\uV_{h+1}^{i-1}-V_{s,a,h}^{i-1}\right)(s,a)\\ 
    &\approx r_h(s,a) + p_h\!\! \underbrace{\left(\frac{H}{n} \! \sum_{i\geq n-n/H}^n \uV_h^{i-1}\right)}_{\approx V_{h,s,a}^n\text{ bias-value function}} \!
    (s,a) \pm \!\!\!\underbrace{\sqrt{ \frac{H^2}{n}}}_{\text{variance term}}\\
    &\quad\pm\underbrace{\sqrt{\frac{H^3}{n}\sum_{i=1}^n p_h(V_{h,s,a}^{n-1} - \uV_h^{n-1})(s,a)\frac{1}{n}}}_{\text{momentum variance term}}\,.
\end{align*}
}%
Note that, we use a \emph{negative} momentum since it allows to put more weight on the recent targets. We thus manage to get the advantages of the two learning rates: use all the samples for small variance and get a bias-value function that only relies on the recent upper-bounds on the optimal value function. This comes only at the cost of an additional momentum variance term that will only influence the dependence on $H$ of the second order term in the regret. Note that here, for sake of simplicity, we used Azuma-Hoeffding inequality which leads to a sub-optimal dependence on the horizon. That is why in the sequel we rather use a Freedman-Bernstein-type inequality (and adapted bonuses) to obtain the optimal dependence on the horizon in the first order term.

Indeed \OptQL by \citet{jin2018is} with Hoeffding-type bonuses has a regret bound of order $\sqrt{H^5SAT}$ with an extra factor $H$ with respect to the lower bound of $\sqrt{H^3SAT}$ (in particular without second order term in~$S^2$). Using Bernstein-type bonuses allows to waive a $\sqrt{H}$ factor in the first-order term. But there is still an extra $\sqrt{H}$ because of the aggressive learning rate of $H/n$ used to deal with the bias issue as described above\footnote{Which will be removed because of the momentum in \OurAlgorithm.}. Note that doing so also introduces a second-order term which is not linear in the number of states $S$, see Table~\ref{tab:upper_bounds}.
This is because in their analysis they need a coarse upper bound on $\uV_h^t-\Vstar_h$ (see Lemma~C.7 in the proof of Lemma C.3 then C.6 by \citet{jin2018is}, such a coarse upper bound is also used by \citet{azar2017minimax}) to link the empirical variance to the true one.
The key point in our analysis is to avoid such an intermediate coarse upper bound which leads inexorably to an extra factor $S$. But instead postpone bounding such quantity to the next step error (we rather control $\uV_h^t-V_h^{\pi^{t+1}}$), see Lemma~\ref{lem:concentration_Q_regret} and Lemma~\ref{lem:bonus_upper_bound}. Indeed, we control $(p^t_h-p_h) \uV_h$ instead of $(p^t_h-p_h) \Vstar_h$ which allows us avoid upper bounding  $\uV^t_h-\Vstar_h$ to build the upper confidence bound (see Lemma~\ref{lem:concentration_Q} and~\ref{lem:bonus_lower_bound}).
But we do not know if it is impossible to build an upper confidence bound (that does not depends on $S$) by only controlling $(p^t_h-p_h) \Vstar_h$.

\subsection{Algorithm}
\label{sec:algorithm_description}
We initialize the upper bounds on the optimal value functions by $\uV_h^0(s) = H$ for all $(s,h)\in\cS\times[H]$. We fix a learning rate $\alpha_h^t(s,a)\geq 0$ a momentum rate $\gamma_h^t(s,a) \geq 0$ such that $\alpha_h^t(s,a)+\gamma_h^t(s,a)\leq 1$. We also consider a bonus function $\beta_h^t(s,a)$. The update of the Q-value for \OurAlgorithm is defined as follows. We update a (biased) estimator of the optimal Q-value function as follow,
\begin{align}
    Q_h^t(s,a) &= \alpha_h^t(s,a) \big(r_h(s,a) + p_h^t \uV_{h+1}^{t-1}(s,a)\big)\nonumber\\&\quad+\gamma_h^t(s,a) p_h^t(\uV_{h+1}^{t-1}-V_{h,s,a}^{t-1})(s,a)\nonumber\\
    &\quad+ \big(1-\alpha_h^t(s,a)\big)Q_h^{t-1}(s,a)\,,
\label{eq:def_Q}
\end{align}
where the bias-value function for state-action $(s,a)$ is defined by, $V_{h,s,a}^0(s')=H$,
{\small
\begin{align}
    V_{h,s,a}^t(s') &= \eta_h^t(s,a) \uV_{h+1}^{t-1}(s')+ \big(1-\eta_h^t(s,a)\big)V_{h,s,a}^{t-1}(s')\,,
\label{eq:def_V}
\end{align}
}%
where we define $\eta_h^t(s,a) = \alpha_h^t(s,a)+  \gamma_h^t(s,a)$. We name this quantity the bias-value function because we will prove that with high probably $  Q_h^t(s,a) \approx r_h(s,a)+ p_h  V_{h,s,a}^t(s,a) $ in Lemma~\ref{lem:concentration_Q_regret} of Appendix~\ref{app:proof_regret_bound}. Then we build upper-confidence bounds on the Q-values by adding a bonus and on the value functions by taking the maximum of the upper-confidence bounds on the Q-values (clipped to be non-increasing)
\[
\uQ_h^t(s,a) = Q_h^t(s,a) +\beta_h^t(s,a)\,,
\]
\[
\uV_h^t(s) =\clip\!\big( \max_{a\in\cA}\uQ_h^t(s,a),0,\uV_h^{t-1}(s)\big)\,,
\]
where the clipping operator is defined as $\clip(x,y,z) = \min(\max(x,y),z)$. We also fix the upper bounds of the value function at step $H+1$ to zero: $\uV_{H+1}^t(s)=0$. Note that $\uQ_h^t(s,a)$ could be negative because of the momentum but it will still be an upper bound on the optimal Q-value with high probability, see Lemma~\ref{lem:optimism}. We also enforce the upper bound on the value function to be non-increasing. We then pick the action greedily with respect to the upper-bounds $\uQ_h^t$. The complete procedure is described in Algorithm~\ref{alg:OurAlgorithm}. We choose (with the convention $0\times \infty =0$ and $1/0 = \infty$)
\begin{align}
\alpha_h^t(s,a) &= \chi_h^t(s,a) \frac{1}{n_h^t(s,a)}\,,\label{eq:def_alpha}\\
\gamma_h^t(s,a) &=   \chi_h^t(s,a) \frac{H}{H+n_h^t(s,a)}\frac{n_h^t(s,a)-1}{n_h^t(s,a)}\label{eq:def_gamma}\,,
\end{align}
for the learning rate and the momentum. Note that in particular it holds $\eta_h^t(s,a)=  \chi_h^t(s,a) (H+1)/(H+n_h^t(s,a))$ which is the learning rate used by \citet{jin2018is}. We can unfold \eqref{eq:def_Q} to obtain explicit formulas for the estimate of the Q-value function when $n_h^t(s,a)>0$:
\begin{align}
  &Q_h^t(s,a)  = r_h(s,a) + \frac{1}{n_h^t(s,a)}\sum_{k=1}^{t} \chi_h^k(s,a)p_h^k \uV_h^{k-1}(s,a) \nonumber\\
  &+ \frac{1}{n_h^t(s,a)}\sum_{k=1}^{t} \chi_h^k(s,a) \rgamma_h^k(s,a) p_h^k (\uV_h^{k-1}-V_{h,s,a}^{k-1})(s,a)\,,
\label{eq:Q_unfold}
\end{align}
where the normalized momentum is definied as
\[
\rgamma_h^k(s,a) =  H \frac{n_h^t(s,a)-1}{n_h^t(s,a)+H} \,.
\] We use a bonus derived from the Bernstein inequality plus a correction term. Precisely if $n_h^t(s,a) = 0$ then $\beta_{h}^t(s,a)=H$ otherwise
\begin{align*}
\beta_h^t(s,a) &= 2\sqrt{W_h^t(s,a)\frac{\zeta}{n_h^t(s,a)}}+53H^3\frac{\zeta\log(T)}{n_h^t(s,a)}\\
\quad+&\sum_{k=1}^t \frac{\chi_h^k(s,a)\rgamma_h^k(s,a)}{H \log(T)  n_h^t(s,a) } p_h^k(V_{h,s,a}^{k-1}-\uV_{h+1}^{k-1})(s,a)\,,
\end{align*}
where $\zeta$ is some exploration threshold that we specify later and $W_h^t$ is a proxy for the variance term
{\small
\[
W_h^t(s,a) \!=\!\!\sum_{k=1}^t \! \frac{\chi_h^k(s,a)}{n_h^t(s,a)}  p_h^k\! \left(\!\uV_{h+1}^{k-1}\!\!-\!\!\sum_{l=1}^t\!  \frac{ \chi_h^l(s,a)}{n_h^t(s,a)}  p_h^l \uV_{h+1}^{l-1} \!\right)^2\!\!\!(s,a)\,.
\]
}
Note that the third term in the bonus will not compensate the momentum because it is $1/(H\log(T))$ times smaller than the momentum term.
\begin{algorithm}[ht]
\centering
\caption{\OurAlgorithm}
\label{alg:OurAlgorithm}
\begin{algorithmic}[1]
   \STATE {\bfseries Initialize:}  For all $(s,a,h)$, $V_{h,s,a}^0=\uV_h^0 = H$ and $Q_h^0 =0$
      \FOR{$t \in[T]$}
        \FOR{$h \in [H]$}
       	   \STATE Play $a_h^t \in \argmax \uQ_h^{t-1}(s_h^t,a)$
           \STATE Observe $s_{h+1}^t\sim p_h(s_h^t,a_h^t)$
           \ENDFOR
        \FOR{all $s,a,h$}
	       \STATE Update $Q_h^t(s,a)$ using Equation~\ref{eq:def_Q}
	        \STATE Update $V_{h,s,a}^t$ for all $s'$ using Equation~\ref{eq:def_V}
	         \STATE $\uQ_h^t(s,a) = Q_h^t(s,a) +\beta_h^t(s,a)$
	          \STATE $\uV_h^t(s) = \clip\!\big(\max_{a\in\cA} \uQ_h^t(s,a), 0, \uV_h^{t-1}(s)\big)$
	    \ENDFOR
	   \ENDFOR
\end{algorithmic}
\end{algorithm}

\subsection{Regret bound}
\label{sec:algorithm_results}
We assume in this section that $T\geq 3$. We fix $ \delta \in(0,1)$ and the exploration threshold
\begin{equation}
  \label{eq:def_zeta}
  \zeta =  \zetaval\,.
\end{equation}
We can now state the main result of the paper which is proved in Appendix~\ref{app:proof_regret_bound}. We sketch the proof in Section~\ref{sec:algorithm_sketch}.
\begin{theorem}
  \label{th:regret_UCBMQ}
  For \OurAlgorithm, with probability at least $1-\delta$
  \[
      R^T \leq    C_1(\delta,T) \sqrt{ H^3 SA  T}+   C_2(\delta,T)  H^4 S A
  \]
where $C_1(\delta,T)  = 126  e^{127}  \log(T) \sqrt{\zeta}$  and $C_2(\delta,T) = 3527 e^{127} \log(T)^2\zeta$.
\end{theorem}
Note that we did not try to optimize the constants $C_1, C_2$. The regret of \OurAlgorithm is thus of order $\tcO\big(\sqrt{H^3SAT} + H^4 SA \big)$ matching the lower bound of $\tcO\big(\sqrt{H^3SAT}\big)$ by  \citet{domingues2020lb} for $T\geq H^5SA$.

\paragraph{Computational complexity.} Note that the update of the upper bounds on the Q-values and value functions and the bias-value functions can be performed online. Indeed at step $h$ and episode $t$, the learning rate $\alpha_h^t(s,a)$ and the momentum rate $\gamma_h^t(s,a)$ equal to zero if $(s,a)\neq (s_h^t,a_h^t)$. Thus the time complexity of \OurAlgorithm is of order $\cO(H(S+A)T)$ for $T$ episodes. This complexity is smaller than the one of model-based algorithms, $\cO(HSAT)$ at best (see \citealt{efroni2019tight}), but is larger than $\cO(HAT)$, the one of model-free algorithms \citep{jin2018is,zhang2020advantage}. The space complexity is $O(HS^2A)$ since we need to store all the bias-value functions, which is the same as the one of model-based algorithms.

\paragraph{Model-free or model-based algorithm.} \OurAlgorithm does not estimate the probability transitions but rather estimates directly the Q-values/values. Therefore \OurAlgorithm can be viewed as a model-free algorithm. On the other hand, the space complexity of \OurAlgorithm is the same as the size of the model $HS^2A$. Thus from a space complexity point of view (see e.g. the definition of model-free algorithms by~\citealt{jin2018is}), \OurAlgorithm is a model-based algorithm.

\paragraph{Comparison with variance reduction methods.}  Building on variance reduction techniques,~\citet{sidford2018near, sidford2018variance} and \citet{zhang2020advantage} propose model-free algorithms that match the problem-independent lower bound for large enough $T$ (see Table~\ref{tab:upper_bounds}).
They use, for some reference value function $V^{\text{ref}}$, the following advantage decomposition of the optimal Q function,
\[
\!\Qstar(s,a)\!=\! r_h(s,a)+p_h V_{h+1}^{\text{ref}}(s,a) + p_h(\Vstar_{h+1} -V_{h+1}^{\text{ref}})(s,a).
\]
To derive their algorithm, they estimate the two expectations above differently. The expectation $p_h V_{h+1}^{\text{ref}}(s,a)$ is estimated using \emph{all the samples}, and the expectation $p_h(\Vstar_{h+1} -V_{h+1}^{\text{ref}})$ is estimated using only the \emph{last $1/H$-fraction} of the samples. The key point is to learn a reference value function $V^{\text{ref}}$ that is close enough to $V^*$ to compensate the smaller number of samples. However, learning such $V^{\text{ref}}$, which is done by using similar update as \eqref{eq:H_over_N_sum}, requires a certain number of episodes, and increases the second term in their analysis. Interestingly, our update~\eqref{eq:def_Q} could be seen as an advantage decomposition: considering~\eqref{eq:Q_unfold}, the bias-value function $V_{h,s,a}^t$ acts as a reference value function. However, contrary to the approach of \citet{zhang2020advantage},  $V_{h,s,a}^t$ is updated continuously as~\eqref{eq:def_V}, instead of being fixed after a ``burn-in'' phase.
\subsection{Proof sketch of Theorem~\ref{th:regret_UCBMQ}}
\label{sec:algorithm_sketch}
We first prove that $\uQ^t$ and $\uV^t$ are indeed upper confidence bounds on the optimal Q-values and the optimal value functions respectively.
\begin{restatable}{lemma}{lemoptimism}
\label{lem:optimism}
On the event $\cE$ that holds with probability $1-\delta$ (see Section~\ref{app:concentration_value}) , $\forall t \in \N, \forall (s,a,h)\in\cS\times\cA\times[H]$ (also for $h = H+1$ for the value function), we have
\begin{align*}
  \uQ_h^t(s,a) \geq \Qstar_h(s,a) \quad\text{and}\quad \uV_h^t(s) \geq \Vstar_h(s)\,.
\end{align*}
\end{restatable}

\textbf{Step 1: Upper-bound $(\uQ_h^{t}-Q_h^{\pi^{t+1}})(s,a)$.} We first upper-bound the difference $(\uQ_h^{t}-Q_h^{\pi^{t+1}})(s,a)$ for a certain state-action pair $(s,a)$. Considering the rewriting~\eqref{eq:Q_unfold} we can apply a Freedman-Bernstein-type inequality (see Appendix~\ref{app:concentration_value}) to replace the sample expectation by the true expectation (see Lemma~\ref{lem:concentration_Q_regret}),
\begin{align*}
  \left|Q_h^t(s,a)- r_h(s,a) - p_h V_{h,s,a}^t(s,a)\right| &\leq b_h^t(s,a)\,,
\end{align*}
where we define, for $\tn_h^t(s,a) = n_h^t(s,a)\land 1$,
{\small
\begin{align*}
  &b_h^t(s,a) \!=\! \sqrt{\frac{4}{\tn_h^t(s,a)}\sum_{k=1}^t  \chi_h^k(s,a) \Var_{p_h}(V_{h+1}^{\pi^k})(s,a) \frac{\zeta}{\tn_h^t(s,a)}}\\
  &+\!\sum_{k=1}^t\! \frac{2 \chi_h^k(s,a) }{H\log(T)\tn_h^t(s,a)} p_h(\uV_{h+1}^{k-1}\!\!-\!\!V_{h+1}^{\pi^k})(s,a)\!+\! 24 H^3\frac{\log(T)\zeta}{\tn_h^t(s,a)}\cdot
\end{align*}
}%
In Lemma~\ref{lem:bonus_upper_bound}, we upper bound the bonus $\beta_h^t(s,a)$ with high probability, with a quantity of the same order as $ b_h^t(s,a)$. Combining these two bounds we obtain
\begin{align}
(\uQ_h^{t}\!-\!Q_h^{\pi^{t+1}})(s,a)\!\leq p_h(V_{h,s,a}^t\!-\!V_{h+1}^{\pi^{t+1}})(s,a)\!+\!6 b_h^t(s,a).
\label{eq:ub_dif_Q_sketch}
\end{align}
\textbf{Step 2: Upper-bound the local optimistic regret.}
Next, we upper-bound the local optimistic regret of state-action $(s,a)$ at step $h$ defined by
\[ \tR_h^T(s,a) = \sum_{t=0}^{T-1} \chi_h^{t+1}(s,a) (\uQ_h^{t}-Q_h^{\pi^{t+1}})(s,a).
\]
We decompose the first term that appears in~\eqref{eq:ub_dif_Q_sketch} by introducing the optimal value function
\begin{align*}
p_h(V_{h,s,a}^t-V_{h+1}^{\pi^{t+1}})(s,a) &= p_h(V_{h,s,a}^t-\Vstar_{h+1})(s,a)\\
&\quad+p_h(\Vstar_{h+1}-V_{h+1}^{\pi^{t+1}})(s,a).
\end{align*}
Then, using Lemma~\ref{lem:properties_weights} from Appendix~\ref{app:count} yields
{\small
\begin{align*}
  &\sum_{t=0}^{T-1}  \chi_h^{t+1}(s,a) p_h(V_{h,s,a}^t-\Vstar_{h+1})(s,a)\\
  &\leq H\! + \!\sum_{k=1}^{T-1} \! \left(\!\sum_{t=k}^{T-1}  \chi_h^{t+1}(s,a) \eta_h^{t,k}(s,a) \!\right) p_h(\uV_{h+1}^{k-1}-\Vstar_{h+1})(s,a)\\
  &\leq H + \left(1+\frac{1}{H}\right)\sum_{t=0}^{T-1}\chi_h^{t+1}(s,a) p_h(\uV_{h+1}^{t-1}-\Vstar_{h+1})(s,a),
\end{align*}
}\noindent we get $V_{h,s,a}^t(s') = \sum_{k=1}^{t} \teta_h^{t,k}(s,a) \uV_{h+1}^{k-1}(s')$ by unfolding~\eqref{eq:def_V}, see~\eqref{eq:V_with_weights} in Appendix~\ref{app:preliminaries}. Combining this inequality with the previous decomposition and using that
$\Vstar_{h+1}\geq V_{h+1}^{\pi^{k+1}}$, we get
\begingroup
\allowdisplaybreaks
\begin{align*}
  &\sum_{t=0}^{T-1}  \chi_h^{t+1}(s,a) p_h(V_{h,s,a}^t-V_{h+1}^{\pi^{t+1}})(s,a) \\
  &\leq \sum_{t=0}^{T-1}  \chi_h^{t+1}(s,a) p_h(\Vstar_{h+1}-V_{h+1}^{\pi^{t+1}})(s,a) +H \nonumber\\
  &\quad + \left(1+\frac{1}{H}\right)\sum_{t=0}^{T-1}\chi_h^{t+1}(s,a) p_h(\uV_{h+1}^{t-1}-\Vstar_{h+1})(s,a)\nonumber\\
  &\leq H+\left(1+\frac{1}{H}\right)\sum_{t=0}^{T-1}\chi_h^{t+1}(s,a) p_h(\uV_{h+1}^{t-1}-V_{h+1}^{\pi^{t+1}})(s,a)\,.
\end{align*}
\endgroup
We can proceed similarly to upper-bound the bonus term using this time Lemma~\ref{lem:sum_1_over_n},~\ref{lem:sum_1_over_n_history} from Appendix~\ref{app:count}, see \eqref{eq:ub_correct_regret}, \eqref{eq:ub_var_regret} and \eqref{eq:ub_bonus_regret} in Appendix~\ref{app:proof_regret_bound}, and get the upper bound on the optimistic local regret,
{\small
\begin{align*}
  &\tR_h^T(s,a) \leq 44 \log(T) \sqrt{\zeta\sum_{t=0}^{T-1} \chi_h^{t+1}(s,a) \Var_{p_h}(V_{h+1}^{\pi^{t+1}})(s,a)} \\
  &\quad+\left(1+\frac{41}{H}\right)\sum_{t=0}^{T-1}\chi_h^{t+1}(s,a) p_h(\uV_{h+1}^{t}-V_{h+1}^{\pi^{t+1}})(s,a)\\
  &\quad+1041 H^3\log(T)^2\zeta\,.
\end{align*}
}%

\paragraph{Step 3: From visit to reach probability.} We denote by $\bp_h^t(s,a)$ respectively $\bp_h^t(s)$ the probability to reach state-action $(s,a)$ respectively state $s$ at step $h$ under the policy~$\pi^t$. We replace the indicator function $\chi_h^t$ by its expectation $\bp_h^t$. Using again an Freedman-Bernstein-type inequality (see Appendix~\ref{app:concentration_count}), from the upper bound on the optimistic local regret above we obtain
{\small
\begin{align}
  \tR_h^T(s,a) &\leq 63 \log(T) \sqrt{\zeta\sum_{t=0}^{T-1} \bp_h^{t+1}(s,a) \Var_{p_h}(V_{h+1}^{\pi^{t+1}})(s,a)} \nonumber\\
  &\quad+\!\left(1\!+\!\frac{83}{H}\right)\!\!\sum_{t=0}^{T-1} \bp_h^{t+1}(s,a) p_h(\uV_{h+1}^{t}\!-\!V_{h+1}^{\pi^{t+1}})(s,a)  \nonumber\\
   &\quad+  1754 H^3 \log(T)^2\zeta\label{eq:ub_local_regret_sketch}\,.
\end{align}
}%

\paragraph{Step 4: Upper-bound the step $h$ optimistic regret.}
We define the regret at step $h$ by
\[
\tR^T_h = \sum_{s\in\cS} \sum_{t=0}^{T-1} \bp_h^{t+1}(s)(\uV_h^{t-1}-V_h^{\pi^{t+1}})(s)\,.
\]
Note that we used the probability to reach the state $s$ rather than the indicator function $\chi_h^t(s)$ above. Using again a Freedman-Bernstein-type inequality (see Appendix~\ref{app:concentration_count}) to upper-bound the reach probability by the indicator function and the definition of $\uV_h^{k}$, we have for $s\in\cS$
{\small
\begin{align*}
  &\sum_{t=0}^{T-1} \bp_h^{t+1}(s) (\uV_h^{t}-V_h^{\pi^{t+1}})(s)\\
  &\leq  \left(1+\frac{1}{H}\right) \sum_{t=0}^{T-1} \chi_h^{t+1}(s) (\uV_h^t(s)-V_h^{\pi^{t+1}})(s) + 19 H^2 \zeta\\
  &\leq \left(1+\frac{1}{H}\right) \sum_{t=0}^{T-1} \chi_h^{t+1}(s)  \pi_{h}^{t+1}(\uQ_h^{k}-Q_h^{\pi^{t+1}})(s)+ 19 H^2 \zeta\,.
\end{align*}
}%
Combining this inequality with~\eqref{eq:ub_local_regret_sketch} then the fact the policies $\pi^t$ are deterministic and Cauchy-Schwarz inequality yield the upper-bound the step $h$ optimistic regret
\begingroup
\allowdisplaybreaks
{\small
\begin{align}
  \tR&^T_h \leq \! \left(\!1\!+\!\frac{1}{H}\!\right)\! \sum_{s,a}\! \sum_{t=0}^{T-1}   \chi_h^{t+1}(s,a)(\uQ_h^{k}\!-\!Q_h^{\pi^{t+1}})(s,a) + 19 H^2 S \zeta\nonumber\\
  &=\!\left(\!1+\frac{1}{H}\!\right)\! \sum_{s,a} \tR_h^T(s,a) + 19 H^2 S \zeta\nonumber\\
  &\leq 126 \log(T) \sqrt{\zeta SA\sum_{s,a}\sum_{t=0}^{T-1} \bp_h^{t+1}(s,a) \Var_{p_h}(V_{h+1}^{\pi^{t+1}})(s,a)}  \nonumber\\
  &\quad+\left(1+\frac{167}{H}\right)\tR_{h+1}^T +  3527 H^3SA \log(T)^2\zeta
  \label{eq:ub_h_regret_sketch}\,,
\end{align}
}%
\endgroup
where in the last inequality we used that
\begin{align*}
\sum_{(s,a,s')\in\cS\times\cA\times\cS}\bp_h^{t+1}(s,a) p_h(s'|s,a) &=  \sum_{s'\in\cS} \bp_{h+1}^{t+1}(s')\,.
\end{align*}

\paragraph{Step 5: Upper-bound the regret.} We upper-bound the Step $1$ regret $\tR_1$. By successively unfolding~\eqref{eq:ub_h_regret_sketch} with the fact that $\tR_{h+1}^T=0$, using the Cauchy-Schwarz inequality and the law of total variance (Lemma~\ref{lem:law_of_total_variance} in Appendix~\ref{app:Bellman_variance}),
{\small
\begin{align*}
  \tR_{1}^T\! &\leq\! \sum_{h=1}^H \!C_1(\delta,T) \sqrt{ SA\!\sum_{s,a}\!\!\sum_{t=0}^{T-1} \bp_h^{t+1}(s,a) \Var_{p_h}\!(V_{h+1}^{\pi^{t+1}})(s,a)}\\
  &\quad+   C_2(\delta,T) H^3 S A\\
  &\leq  C_1(\delta,T) \sqrt{\! SAH\sum_{s,a,h}\sum_{t=0}^{T-1} \!\bp_h^{t+1}(s,a) \Var_{p_h}(V_{h+1}^{\pi^{t+1}})(s,a)}\\
  &\quad+ C_2(\delta,T)   H^4 S A\\
  &\leq C_1(\delta,T) \sqrt{H^3 SA T}+ C_2(\delta,T) H^4 S A \,.
\end{align*}
}%
It remains to relate the opstimistic regret with the regret. Thanks to Lemma~\ref{lem:optimism} we have
\begin{align*}
  \Vstar_1(s_1)-V_h^{\pi^{t+1}}(s_1) \leq \uV_1^t(s_1) - V_1^{\pi^{t+1}}(s_1)\,,
\end{align*}
which allows us to conclude
\begin{align*}
  R^T \leq \tR_1^T \leq  C_1(\delta,T) \sqrt{SAH^3T} + C_2(\delta,T) S A H^4\,.
\end{align*}

\section{Experiments}
\label{sec:experiments}

In this section, we present a numerical simulation to illustrate the benefits of not forgetting the targets in \OurAlgorithm. We compare \OurAlgorithm to the following baselines:
\begin{enumerate*}[label=(\roman*)]
	\item \UCBVI~\cite{azar2017minimax};
	\item \OptQL~\cite{jin2018is}, and
	\item \UCBVIgreedy, a version of \UCBVI using real-time dynamic programming~\cite{efroni2019tight}.
\end{enumerate*}
We use a grid-world environment with $50$ states $(i, j) \in [10]\times[5]$ and $4$ actions (left, right, up and down). When taking an action, the agent moves in the corresponding direction with probability $1-\epsilon$, and moves to a neighbor state at random with probability $\epsilon$. The starting position is $(1, 1)$. The reward equals to $1$ at the state $(10, 5)$ and is zero elsewhere.\footnote{The code to reproduce the experiments is available on \href{https://github.com/omardrwch/ucbmq_code}{GitHub}, and uses the \href{https://github.com/rlberry-py/rlberry}{\texttt{rlberry}} library \cite{rlberry}. }

Using different exploration bonuses (e.g., by changing the multiplicative constants) can result in drastically different regrets empirically. In order to fairly compare the algorithmic ideas of \OurAlgorithm to the baselines, we use the \emph{same exploration bonus} for all the algorithms, given by:
{\small
\begin{align*}
	\beta_h^t(s,a) =
	\min\left(
	\sqrt{\frac{1}{n_h^t(s,a)}} + \frac{H-h+1}{n_h^t(s,a)}, H-h+1
	\right)\,.
\end{align*}
}
Although the confidence intervals required by the algorithms are not always satisfied with this bonus, they hold for $n_h^t(s,a) = 0$ (resulting in $\beta_h^t(s,a)=H-h+1$), so that this choice does not hurt the initial exploration. When $n_h^t(s,a) > 0$, the bonus behaves as a simplified version of the Bernstein-type bonuses used in different algorithms.

\begin{figure}[ht]
	\centering
	\includegraphics[width=8cm]{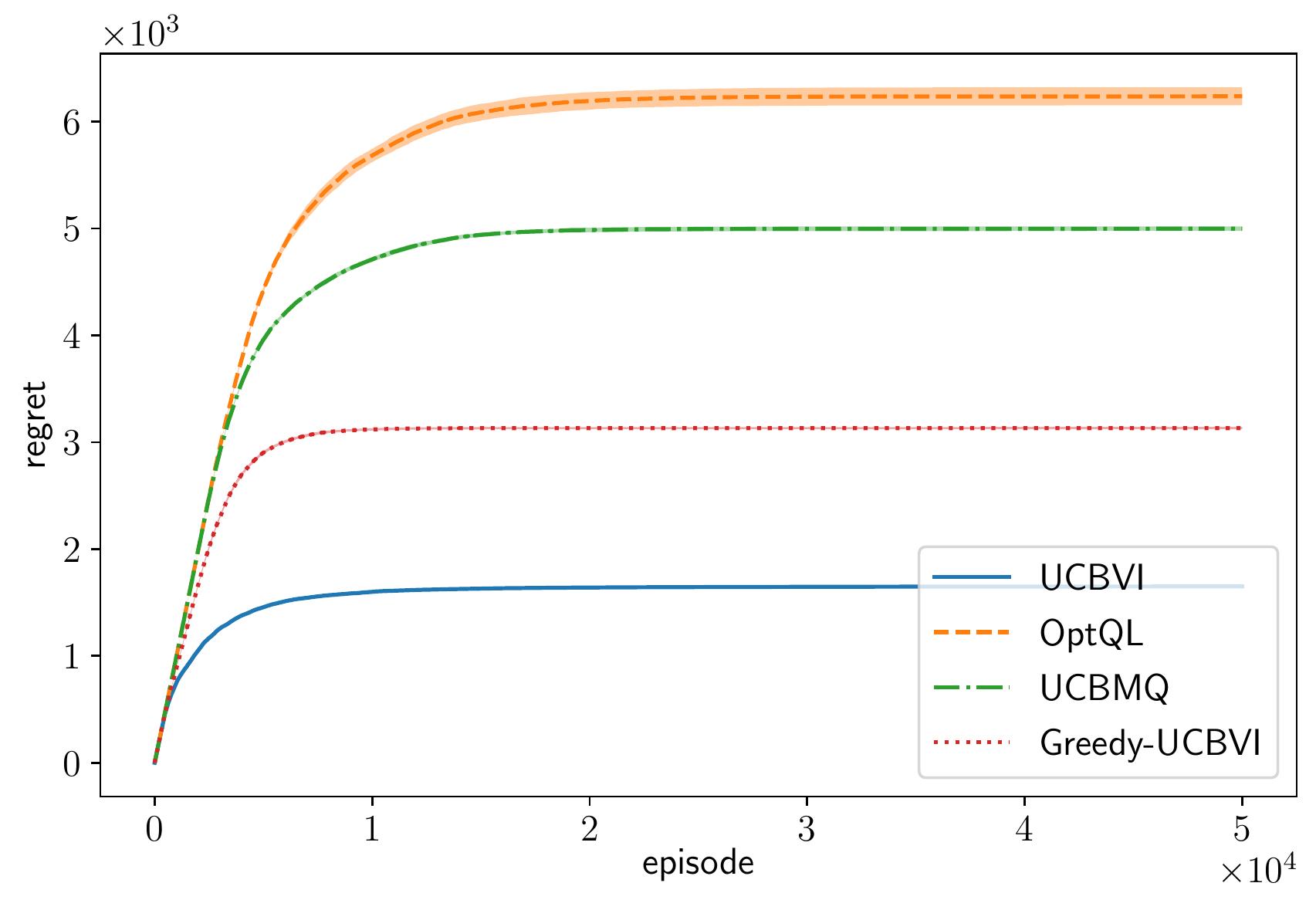}
	\caption{Regret of \OurAlgorithm compared to baselines, for $H=100$ and transition noise $\epsilon=0.15$. Average over $8$ runs.}
	\label{fig:low-noise}
\end{figure}

In Figure~\ref{fig:low-noise}, we observe that \OurAlgorithm outperforms \OptQL in our experiments, whereas the only differences in the implementations of the two algorithms are the learning rates and the momentum term used by \OurAlgorithm (since the bonuses were kept identical). This illustrates the potential gain in sample efficiency enabled by not forgetting the targets.

We also observe that, in this simulation, \OurAlgorithm has a larger regret than \UCBVI and \UCBVIgreedy, which are model-based algorithms using empirical estimates of the transitions probabilities and planning. It is not surprising since explicitly using a model and backward induction allows new information to be more quickly propagated to the value function computed by the algorithms. \UCBVI performs \emph{full planning} after each episode. \UCBVIgreedy does \emph{1-step} planning, propagating information more quickly than \OurAlgorithm, but more slowly than \UCBVI, which explains the results in Figure~\ref{fig:low-noise}. However, current regret bounds for model-based algorithms, such as \UCBVI, still feature a second order term scaling with $S^2$ (see Table~\ref{tab:upper_bounds}): an interesting open question is whether a bound scaling linearly with $S$ can be obtained when a transition model is used.

\section{Conclusion}
\label{sec:conclusion}
We studied regret minimization in tabular, non-stationary, episodic MDPs. For this settings, we provided an algorithm a regret bound that is optimal in a problem-independent sense for a large enough number of episodes~$T$ and such that the \emph{second-order term in the regret bound scales only linearly with the number of states} $S$. Our result rises following interesting open questions for a further research.

\paragraph{Problem-independent optimal regret.} We conjecture that the optimal problem-independent regret is $\cO(\sqrt{H^3SAT}+H^2SA)$. This conjecture is coherent with the one of \citet{wang2020sample} for PAC problem-independent optimal sample complexity if we do not assume that the sum of the rewards along any trajectory is smaller than $1$. In particular, it is not clear how to obtain a better dependency on the horizon $H$ in the second-order term, while being only linear in $S$. For \OurAlgorithm we have an extra $H$ factor in the second-order term in comparison to the regret bound of \UCBVI. This is due to the momentum rate $\gamma$ which scales with $H$ (Equation~\ref{eq:def_gamma}). This scaling seems necessary to refrain from getting an extra $H$ factor at the first-order term and it is unclear how to avoid it. Note that if our conjecture for the optimal problem-independent regret is true, the regret bounds for the model-based algorithms (e.g., \UCBVI, see Table~\ref{tab:upper_bounds})  would  be sub-optimal in $H$ in the second-order term.

\paragraph{Dependency on $S$ for model-based algorithms.} Even if \OurAlgorithm could be considered as a model-based algorithm (Section~\ref{sec:algorithm_description}) it relies on the model-free Q-learning algorithm. This is the main reason behind obtaining a linear dependence on the size of the state space $S$ in the second-order term. As explained in Section~\ref{sec:intro}, it is not clear to obtain similar bounds for model-based algorithm but experimentally they perform better, see Section~\ref{sec:experiments}. Interestingly with access to a generative model, \citet{szita2010model} managed to get rid of the extra factor $S$ for PAC-MDP complexity \citep{kakade2003thesis}.

\paragraph{Computational complexity.} As we need to maintain a separate bias-value function for each state-action pairs, \OurAlgorithm has a larger complexity both in time and space than the algorithms of~\citet{jin2018is} and \citet{zhang2020advantage}. It is not clear how to obtain an algorithm with the same guarantees as \OurAlgorithm while having a space complexity of $\cO(HSA)$ and a time complexity of $\cO(HT)$.

\section*{Acknowledgements}
The research presented was supported by European CHIST-ERA project DELTA. Pierre  Ménard  is supported by the SFI Sachsen-Anhalt for the project RE-BCI ZS/2019/10/102024 by the Investitionsbank Sachsen-Anhalt.
\bibliographystyle{icml2021_style/icml2021.bst}
\bibliography{library}

\newpage
\appendix
\onecolumn

\part{Appendix}
\parttoc
\newpage
\section{Notations}
\label{app:table_notation}

\begin{table}[h]
	\centering
	\caption{Table of notation}
	\begin{tabular}{@{}l|l@{}}
		\toprule
		\thead{Notation} & \thead{Meaning} \\ \midrule
	$\cS$ & state space of size $S$\\
	$\cA$ & action space of size $A$\\
	$H$ & length of one episode\\
	$T$ & number of episodes\\
	$r_h(s,a)$ & reward \\
	$p_h(s'|s,a)$ & probability transition \\
	$p_h^t(s'|s,a)$ & Dirac distribution  $(p_h^t f)(s,a) = f(s_{h+1}^t)$\\
	$\chi_h^t(s)$	& indicator function $\chi_h^t(s) = \ind_{\{s_h^t= s\}}$\\
	$\chi_h^t(s,a)$ & indicator function $ \chi_h^t(s,a)= \ind_{\{(s_h^t,a_h^t) = (s,a)\}}$\\
	$n_h^t(s,a)$ & number of visits of state-action $n_h^t(s,a) = \sum_{k = 1}^t \chi_h^t(s,a)$\\
	$\tn_h^t(s,a)$ & maximum $\tn_h^t(s,a)=\max(n_h^t(s,a),1)$\\
	$Q_h^t(s,a)$ & estimate of the Q-value, see~\eqref{eq:def_Q}\\
  $\uQ_h^t(s,a)$ & upper bound on the optimal Q-value\\
  $\uV_h^t(s)$ & upper bound on the optimal values\\
  $V_{s,a,h}^t$ & biased-value function, see~\eqref{eq:def_V}\\
	$\alpha_h^t(s,a)$ & learning rate, see~\eqref{eq:def_alpha}\\
	$\gamma_h^t(s,a)$ & momentum rate, see~\eqref{eq:def_gamma}\\
	$\eta_h^t(s,a)$ & learning rate of the bias-value function, $\eta_h^t(s,a) = (\alpha_h^t+\gamma_h^t)(s,a)$\\
	$\beta_h^t(s,a)$ & bonus, see Section~\ref{sec:algorithm_description}\\
	$\zeta$ & exploration rate, see~\eqref{eq:def_zeta}\\
	$\bp_h^t(s,a)$ & probability to visit $(s,a)$ at step $h$ under $\pi_h^t$\\
	$\bp_h^t(s)$ & probability to visit $s$ at step $h$ under $\pi_h^t$\\
		\bottomrule
	\end{tabular}
\end{table}

\newpage
\section{Preliminaries}
\label{app:preliminaries}

We can unfold \eqref{eq:def_Q} to obtain explicit formulas for the estimate of the Q-value function when $n_h^t(s,a)>0$:
\begin{align*}
  Q_h^t(s,a)  &= r_h(s,a) + \frac{1}{n_h^t(s,a)}\sum_{k=1}^{t} \chi_h^k(s,a)\left( p_h^k \uV_{h+1}^{k-1}(s,a) + \rgamma_h^k(s,a) p_h^k (\uV_{h+1}^{k-1}-V_{h,s,a}^{k-1})(s,a)  \right)
\end{align*}
where we defined the normalized momentum
\[
\rgamma_h^t(s,a) =  H \frac{n_h^t(s,a)-1}{n_h^t(s,a)+H} \,.
\]
Note that in particular $0 \leq\rgamma_h^t(s,a) \leq H$. We can do the same with \eqref{eq:def_V} for the bias value function of state-action $(s,a)$ when $n_h^t(s,a)>0$
\begin{align}
  V_{h,s,a}^{t}(s') &= \frac{1}{n_h^t(s,a)}\sum_{k=1}^{t} \chi_h^k(s,a)\left( \uV_{h+1}^{k-1}(s') + \rgamma_h^k(s,a) (\uV_{h+1}^{k-1}-V_{h,s,a}^{k-1})(s')  \right) \label{eq:V_expand}\\
  &= \eta_h^t(s,a) \uV_{h+1}^{t-1}(s') + (1-\eta_h^t(s,a)) V_{h,s,a}^{t-1}(s')\nonumber\\
  &=  \sum_{k=1}^{t} \teta_h^{t,k}(s,a) \uV_{h+1}^{k-1}(s')\label{eq:V_with_weights}
\end{align}
where we defined the cumulative weights
\[
\teta_h^{t,k}(s,a)  =  \eta_h^k(s,a) \prod_{l=k+1}^t \big(1-\eta_h^l(s,a)\big)\quad\text{recalling}\quad  \eta_h^t(s,a) = \chi_h^t(s,a) \frac{H+1} {H+n_h^t(s,a)}\,.
\]
We regroup in the following lemma properties on the different value functions that hold almost surely.
\begin{lemma}
\label{lem:determ_prop_values} For all $(s, s', t)$, it holds almost surely:
\begin{itemize}
  \item the sequence $(\uV_h^t(s))_{t\geq 0}$ is non-increasing,
  \item $0 \leq \uV_h^t(s) \leq H$,
  \item $ \uV_{h+1}^t(s')\leq  V_{h,s,a}^t(s')\leq H$.
\end{itemize}
\end{lemma}
\begin{proof}
	The fact that  $(\uV_h^t(s))_{t\geq 0}$ is non increasing comes directly by construction
	$$\uV_h^t(s) =\clip\!\big( \max_{a\in\cA}\uQ_h^t(s,a),0,\uV_h^{t-1}(s)\big) \leq \uV_h^{t-1}(s).$$
	To prove that $0 \leq \uV_h^t(s) \leq H$, we proceed by induction. The algorithm initializes $\uV_h^{0}(s)$, hence the claim is satisfied by $t=0$. Assuming that $0 \leq \uV_h^{t-1}(s) \leq H$, the equation above implies that this is also satisfied by $\uV_h^t(s)$. For the third point, we have
	$$
	\uV_{h+1}^t(s')
	\leq \min_{k\in\{1, \ldots t \}}\uV_{h+1}^{k-1}(s')
	\leq \sum_{k=1}^{t} \teta_h^{t,k}(s,a) \uV_{h+1}^{k-1}(s')
	\leq \sum_{k=1}^{t} \teta_h^{t,k}(s,a) H \leq H
	$$
	and we use the fact that $\sum_{k=1}^{t} \teta_h^{t,k}(s,a) \uV_{h+1}^{k-1}(s')
	= V_{h,s,a}^{t}(s')$.
\end{proof}

 \newpage
\section{Concentration events}
\label{app:concentration}

\subsection{From sample mean to expectation}
\label{app:concentration_value}

We define the favorable events $\cE^{\mathrm{v}_1}$ and $\cE^{\mathrm{v}_2}$ where we control two martingales involving the moment of order 1 and 2 of the upper bounds on the value function at the next step. We also define $\cE^{\mathrm{m}_w}, \cE^{\mathrm{m}}$ where we control the martingale of the momentum term with and without weights, precisely
\begin{align*}
  \cE^{\mathrm{v}_1} &\triangleq \Bigg\{\forall t \in \N, \forall h \in [H], \forall (s,a)\in\cS\times\cA:\\
  &\qquad\left|\sum_{k=1}^t \chi_h^k(s,a) (p_h^k-p_h)\uV_{h+1}^{k-1}(s,a)\right| \leq \sqrt{2 \zeta \sum_{k=1}^t  \chi_h^k(s,a) \Var_{p_h}(\uV_h^{k-1})(s,a)}+6H\zeta\Bigg\}\,,\\
  \cE^{\mathrm{v}_2} &\triangleq \Bigg\{\forall t \in \N, \forall h \in [H], \forall (s,a)\in\cS\times\cA:\\
  &\qquad\left|\sum_{k=1}^t \chi_h^k(s,a) (p_h^k-p_h)(\uV_{h+1}^{k-1})^2(s,a)\right| \leq \sqrt{8 H^2 \zeta \sum_{k=1}^t  \chi_h^k(s,a) \Var_{p_h}(\uV_{h+1}^{k-1})(s,a)}+12 H^2 \zeta\Bigg\}\\
  \cE^{\mathrm{m}_w} &\triangleq \Bigg\{\forall t \in \N, \forall h \in [H], \forall (s,a)\in\cS\times\cA:\\
  &\quad\left|\sum_{k=1}^t \chi_h^k(s,a) \rgamma_h^k(s,a) (p_h^k-p_h)(\uV_{h+1}^{k-1}-V_{h,s,a}^{k-1})(s,a)\right| \leq \\
  &\qquad\qquad\qquad\qquad\qquad\qquad\qquad\qquad\qquad\sqrt{2 \zeta \sum_{k=1}^t  \chi_h^k(s,a) \rgamma_h^k(s,a)^2 \Var_{p_h}(\uV_{h+1}^{k-1}-V_{h,s,a}^{k-1})(s,a)} +6 H^2 \zeta\Bigg\}\\
  \cE^{\mathrm{m}} &\triangleq \Bigg\{\forall t \in \N, \forall h \in [H], \forall (s,a)\in\cS\times\cA:\\
  &\qquad\left|\sum_{k=1}^t \chi_h^k(s,a) (p_h^k-p_h)(\uV_{h+1}^{k-1}-V_{h,s,a}^{k-1})(s,a)\right| \leq \sqrt{2 \zeta \sum_{k=1}^t  \chi_h^k(s,a) \Var_{p_h}(\uV_{h+1}^{k-1}-V_{h,s,a}^{k-1})(s,a)}+6 H \zeta\Bigg\}\,.
\end{align*}
We define $\cE = \cE^{\mathrm{v}_1} \cap \cE^{\mathrm{v}_2} \cap \cE^{\mathrm{m}_w} \cap \cE^{\mathrm{m}}$ the intersection of these events where the optimism will be true. This event holds with high probability.
\begin{lemma}
\label{lem:event_optimism}
For the choice
\[
\zeta  = \log\left(32e(2T+1)/\delta\right),
\]
it holds $\P(\cE)\geq 1-\delta/2$.
\end{lemma}

\begin{proof}
Thanks to the choice of $\zeta$ and Theorem~\ref{th:bernstein} we have
\[
\P({(\cE^{\mathrm{v}}})^c) \leq \frac{\delta}{8}\,,\qquad \P({(\cE^{\mathrm{v}_2}})^c) \leq \frac{\delta}{8}\,,\qquad \P({(\cE^{\mathrm{m}_w}})^c) \leq \frac{\delta}{8}\,, \qquad \P({(\cE^{\mathrm{m}}})^c) \leq \frac{\delta}{8}\,.
\]
For the second event, note that thanks to the Freedman-Bernstein-type inequality (Theorem~\ref{th:bernstein}) with probability at least $1-\delta/6$ it holds
\begin{align*}
&\forall t \in \N, \forall h \in [H], \forall (s,a)\in\cS\times\cA:\\
&\quad\left|\sum_{k=1}^t \chi_h^k(s,a) (p_h^k-p_h)(\uV_{h+1}^{k-1})^2(s,a)\right| \leq \sqrt{2 \zeta \sum_{k=1}^t  \chi_h^k(s,a) \Var_{p_h}(\uV_{h+1}^{k-1})^2(s,a)}+12 H^2 \zeta\,.
\end{align*}
Thanks to Lemma~\ref{lem:switch_variance} we know that $\Var_{p_h}(\uV_{h+1}^{k-1})^2(s,a) \leq 2H^2 \Var_{p_h}(\uV_{h+1}^{k-1})(s,a)$ and consequently the last event holds with high probability $\P((\cE^{\mathrm{v}_2})^c) \leq \frac{\delta}{6}$. A union bound allows us to conclude.
\end{proof}

We can get a confidence of order $1/n$ at the price of a constant term when the variance is not important in the concentration inequalities of event $\cE$.
\begin{lemma}
  \label{lem:concentration_without_variance}
  On the event $\cE$, $\forall t \in \N, \forall h \in [H], \forall (s,a)\in\cS\times\cA$,  it holds
  \begin{align}
    \left|\sum_{k=1}^t \chi_h^k(s,a) \rgamma_h^k(s,a) (p_h^k-p_h)(\uV_{h+1}^{k-1}-V_{h,s,a}^{k-1})(s,a)\right| &\leq \frac{1}{4H\log(T)}\sum_{k=1}^t  \chi_h^k(s,a)  \rgamma_h^k(s,a) p_h(V_{h,s,a}^{k-1}-\uV_{h+1}^{k-1})(s,a)\nonumber\\
    &\qquad + 14 H^3 \log(T)\zeta \label{eq:concentration_momentum}\\
    \left|\sum_{k=1}^t \chi_h^k(s,a) (p_h^k-p_h)(\uV_{h+1}^{k-1}-V_{h,s,a}^{k-1})(s,a)\right| &\leq \frac{1}{4} \sum_{k=1}^t \chi_h^k(s,a) p_h(V_{h,s,a}^{k-1}-\uV_{h+1}^{k-1})(s,a) + 14 H\zeta \label{eq:concentration_momentum_bonus}\\
    \left|\sum_{k=1}^t \chi_h^k(s,a) (p_h^k-p_h)(\uV_{h+1}^{k-1})^2(s,a)\right| &\leq  \frac{1}{4}\sum_{k=1}^t  \chi_h^k(s,a) \Var_{p_h}(\uV_{h+1}^{k-1})(s,a)+44 H^2 \zeta \label{eq:concentration_order_2}
\end{align}
\end{lemma}
\begin{proof}
For \eqref{eq:concentration_momentum} we use that event $\cE^{\mathrm{m}_w}$ holds,$\rgamma_h^t(s,a)\leq H$, $0 \leq V_{h,s,a}^{k-1}-\uV_{h+1}^{k-1} \leq H$ and $\sqrt{xy}\leq x+y$,
\begin{align*}
  \left|\sum_{k=1}^t \chi_h^k(s,a) \rgamma_h^k(s,a) (p_h^k-p_h)(\uV_{h+1}^{k-1}-V_{h,s,a}^{k-1})(s,a)\right| &\leq \sqrt{2 H\zeta \sum_{k=1}^t  \chi_h^k(s,a) \rgamma_h^k(s,a) \Var_{p_h}(\uV_{h+1}^{k-1}-V_{h,s,a}^{k-1})(s,a)}\\
  &\quad+6 H^2 \zeta\\
  &\leq \sqrt{2 \zeta H^2 \sum_{k=1}^t  \chi_h^k(s,a)\rgamma_h^k(s,a) p_h(V_{h,s,a}^{k-1}-\uV_{h+1}^{k-1})(s,a)}\\
  &\quad+6 H^2 \zeta\\
  &\leq \frac{1}{4 H \log(T)}\sum_{k=1}^t  \chi_h^k(s,a) \rgamma_h^k(s,a) p_h(V_{h,s,a}^{k-1}-\uV_{h+1}^{k-1})(s,a)\\
  &\qquad + 14 H^3 \log(T)\zeta\,.
\end{align*}
For~\eqref{eq:concentration_momentum_bonus} we proceed similarly as above knowing that $\cE^{\mathrm{m}}$ holds
\begin{align*}
    \left|\sum_{k=1}^t \chi_h^k(s,a) (p_h^k-p_h)(\uV_{h+1}^{k-1}-V_{h,s,a}^{k-1})(s,a)\right| &\leq  \sqrt{2 \zeta \sum_{k=1}^t  \chi_h^k(s,a) \Var_{p_h}(\uV_{h+1}^{k-1}-V_{h,s,a}^{k-1})(s,a)}+6 H^2 \zeta\\
    &\leq\sqrt{2 H\zeta \sum_{k=1}^t  \chi_h^k(s,a) p_h(\uV_{h+1}^{k-1}-V_{h,s,a}^{k-1})(s,a)}+6 H \zeta\\
    &\leq  \frac{1}{4} \sum_{k=1}^t \chi_h^k(s,a) p_h(V_{h,s,a}^{k-1}-\uV_{h+1}^{k-1})(s,a) + 14 H^2\zeta\,.
\end{align*}
And for~\eqref{eq:concentration_order_2} we use event  $\cE^{\mathrm{v}_2}$:
\begin{align*}
  \left|\sum_{k=1}^t \chi_h^k(s,a) (p_h^k-p_h)(\uV_{h+1}^{k-1})^2(s,a)\right| &\leq \sqrt{8 H^2 \zeta \sum_{k=1}^t  \chi_h^k(s,a) \Var_{p_h}(\uV_{h+1}^{k-1})(s,a)}+12 H^2 \zeta\\
  &\leq \frac{1}{4} \sum_{k=1}^t  \chi_h^k(s,a) \Var_{p_h}(\uV_{h+1}^{k-1})(s,a) + 44 H^2 \zeta\,.
\end{align*}
\end{proof}

\subsection{From empirical visits to reach probability}
\label{app:concentration_count}
We also define an event where we replace the indicator function to visit a state-action or a state by its expectation. We define by $\bar{p}_h^t(s,a)$ and $\bar{p}_h^t(s)$ the probabilities to reach state-action $(s,a)$ and state $s$, respectively, at step $h$ under the policy $\pi^t$. Precisely we define the event $\cG^{\mathrm{var}}$,   $\cG^{\mathrm{v}_1}$, $\cG^{\mathrm{v}_1}$
where we replace $\chi_h^t(s,a)$ or  $\chi_h^t(s)$ by its expectation when it is multiplied by a predictable quantity,
\begin{align*}
  \cG^{\mathrm{var}} &\triangleq \Bigg\{\forall t \in \N, \forall h \in [H], \forall (s,a)\in\cS\times\cA:\\
  &\qquad\left|\sum_{k=1}^t (\chi_h^k-\bar{p}_h^k)(s,a) \Var_{p_h}(V_{h+1}^{\pi^t})(s,a)\right| \leq \sum_{k=1}^t \bar{p}_h^k(s,a) \Var_{p_h}(V_{h+1}^{\pi^t})(s,a)+8H^2\zeta \,,\\
  \cG^{\mathrm{v}_1} &\triangleq \Bigg\{\forall t \in \N, \forall h \in [H], \forall (s,a)\in\cS\times\cA:\\
  &\qquad\left|\sum_{k=1}^t (\chi_h^k-\bar{p}_h^k)(s,a) p_h(\uV_{h+1}^{k-1}-V_{h+1}^{\pi^t})(s,a)\right| \leq \frac{1}{4H}\sum_{k=1}^t \bar{p}_h^k(s,a) p_h|\uV_{h+1}^{k-1}-V_{h+1}^{\pi^t}|(s,a) + 14H^2 \Bigg\}\,,\\
  \cG^{\mathrm{v}_2} &\triangleq \Bigg\{\forall t \in \N, \forall h \in [H], \forall s\in\cS:\\
  &\quad\left|\sum_{k=1}^t (\chi_h^k-\bar{p}_h^k)(s) (\uV_{h}^{k-1}-V_{h}^{\pi^t})(s)\right| \leq \frac{1}{4H}\sum_{k=1}^t \bar{p}_h^k(s) |\uV_{h+1}^{k-1}-V_{h+1}^{\pi^t}|(s)+14H^2\zeta\Bigg\}\,.
\end{align*}
We define $\cG = \cG^{\mathrm{var}} \cap \cG^{\mathrm{v}_1} \cap \cG^{\mathrm{v}_2}$ the intersection of these events and the previous event $\cE$. This event holds with high probability.
\begin{lemma}
\label{lem:event_count}
For the choice
\[
\zeta  =  \zetaval,
\]
it holds $\P(\cG)\geq 1-\delta/2$.
\end{lemma}
\begin{proof}
Thanks to Theorem~\ref{th:bernstein}, with probability at $1-\delta/8$, for all $s,a,h,t$ we have, using that for a Bernoulli distribution of parameter $X\sim \Ber(q)$ its variance is upper-bounded by $\Var(X) = q(1-q) \leq q$ and $\sqrt{xy}\leq x+y$,
\begin{align*}
  \left|\sum_{k=1}^t (\chi_h^k-\bar{p}_h^k)(s,a) \Var_{p_h}(V_{h+1}^{\pi^t})(s,a)\right| &\leq \sqrt{ 2\zeta \sum_{k=1}^t \bar{p}_h^k(s,a) \Var_{p_h}(V_{h+1}^{\pi^t})(s,a)^2} + 6 \zeta H^2 \zeta\\
  &\leq \sum_{k=1}^t \bar{p}_h^k(s,a) \Var_{p_h}(V_{h+1}^{\pi^t})(s,a) + 8 \zeta H^2\,.
\end{align*}
Thus we know that $\P\big( (\cG^{\mathrm{var}})^c \big) \leq \delta/8$. Similarly for the second event, thanks to Theorem~\ref{th:bernstein}, with probability at $1-\delta/8$, for all $s,a,h,t$ we obtain
\begin{align*}
  \left|\sum_{k=1}^t (\chi_h^k-\bar{p}_h^k)(s,a) p_h(\uV_{h+1}^{k-1}-V_{h+1}^{\pi^t})(s,a) \right| &\leq \sqrt{ 2\zeta \sum_{k=1}^t \bar{p}_h^k(s,a) p_h(\uV_{h+1}^{k-1}-V_{h+1}^{\pi^t})(s,a)^2 } + 6 \zeta H^2 \zeta\\
  &\leq \frac{1}{4H}\sum_{k=1}^t \bar{p}_h^k(s,a) p_h|\uV_{h+1}^{k-1}-V_{h+1}^{\pi^t}|(s,a) + 14 H^2 \zeta\,.
\end{align*}
Thus it holds $\P\big( (\cG^{\mathrm{v}_1})^c \big) \leq \delta/8$. We proceed in the same way for the last event. Thanks to Theorem~\ref{th:bernstein}, with probability at $1-\delta/8$, for all $s,a,h,t$ we have
\begin{align*}
  \left|\sum_{k=1}^t (\chi_h^k-\bar{p}_h^k)(s) (\uV_{h+1}^{k-1}-V_{h+1}^{\pi^t})(s) \right| &\leq \sqrt{ 2\zeta \sum_{k=1}^t \bar{p}_h^k(s) (\uV_{h+1}^{k-1}-V_{h+1}^{\pi^t})(s)^2 } + 6 \zeta H^2 \zeta\\
  &\leq \frac{1}{4H}\sum_{k=1}^t \bar{p}_h^k(s) p_h|\uV_{h+1}^{k-1}-V_{h+1}^{\pi^t}|(s) + 14 H^2 \zeta\,.
\end{align*}
Thus it holds $\P\big( (\cG^{\mathrm{v}_2})^c \big) \leq \delta/8$.
An union bound allows us to conclude.
\end{proof}

\subsection{The favorable event}
\label{app:concentration_master}
We define the event $\cD= \cE\cap \cG$ as the intersection of the event $\cE$ where the optimism will hold and $\cG$ where we can relate the empirical number of visits of a state-action to the probability of visit. In particular the regret bound will be true on this event which holds with high probability.
\begin{lemma}
\label{lem:master_event}
For the choice
\[
\zeta  = \zetaval,
\]
it holds $\P(\cD)\geq 1-\delta$.
\end{lemma}
\begin{proof}
  This is a simple consequence of Lemma~\ref{lem:event_optimism} and Lemma~\ref{lem:event_count}.
\end{proof}

\subsection{Deviation inequality for bounded distributions}
Below, we reproduce the self-normalized Freedman-Bernstein-type inequality by \citet{domingues2020regret}. Let $(Y_t)_{t\in\N^\star}$, $(w_t)_{t\in\N^\star}$ be two sequences of random variables adapted to a filtration $(\cF_t)_{t\in\N}$. We assume that the weights are in the unit interval $w_t\in[0,1]$ and predictable, i.e. $\cF_{t-1}$ measurable. We also assume that the random variables $Y_t$  are bounded $|Y_t|\leq b$ and centered $\EEc{Y_t}{\cF_{t-1}} = 0$.
Consider the following quantities
\begin{align*}
		S_t \triangleq \sum_{s=1}^t w_s Y_s, \quad V_t \triangleq \sum_{s=1}^t w_s^2\cdot\EEc{Y_s^2}{\cF_{s-1}}, \quad \mbox{and} \quad W_t \triangleq \sum_{s=1}^t w_s
\end{align*}
and let $h(x) \triangleq (x+1) \log(x+1)-x$ be the Cramér transform of a Poisson distribution of parameter~1.

\begin{theorem}[Bernstein-type concentration inequality]
  \label{th:bernstein}
	For all $\delta >0$,
	\begin{align*}
		\PP{\exists t\geq 1,   (V_t/b^2+1)h\left(\!\frac{b |S_t|}{V_t+b^2}\right) \geq \log(1/\delta) + \log\left(4e(2t+1)\!\right)}\leq \delta.
	\end{align*}
  The previous inequality can be weakened to obtain a more explicit bound: if $b\geq 1$ with probability at least $1-\delta$, for all $t\geq 1$,
 \[
 |S_t|\leq \sqrt{2V_t \log\left(4e(2t+1)/\delta\right)}+ 3b\log\left(4e(2t+1)/\delta\right)\,.
 \]
\end{theorem}

 \newpage
\section{Optimism}
\label{app:optimism}

We will prove in the next lemma that $Q_h^t(s,a) \approx r_h(s,a) + p_h V_{h,s,a}^t(s,a)$ thus the bias of our estimator will be controlled by the bias of $ V_{h,s,a}^t $ with respect to $\Vstar_{h+1}$.
\begin{lemma}
  \label{lem:concentration_Q}
  On the event $\cE$, $\forall t \in \N, \forall h \in [H], \forall (s,a)\in\cS\times\cA$, if $n_h^t(s,a)>0$, it holds
  \begin{align*}
    \left|Q_h^t(s,a)- r_h(s,a) - p_h V_{h,s,a}^t(s,a)\right| &\leq \sqrt{\frac{2}{n_h^t(s,a)}\sum_{k=1}^t  \chi_h^k(s,a) \Var_{p_h}(\uV_{h+1}^{k-1})(s,a) \frac{\zeta}{n_h^t(s,a)}}+ 20 H^3 \frac{\zeta\log(T)}{n_h^t(s,a)}\\
    &\quad+\frac{1}{4\log(T) H n_h^t(s,a)}\sum_{k=1}^t  \chi_h^k(s,a) \rgamma_h^k(s,a) p_h(V_{h,s,a}^{k-1}-\uV_{h+1}^{k-1})(s,a)\,.
  \end{align*}
\end{lemma}
\begin{proof}
Thanks to the definition of the bias-value function $V_{h,s,a}^t$ we have
\begin{align*}
    \left|Q_h^t(s,a)- r_h(s,a) - p_h V_{h,s,a}^t(s,a)\right| &\leq \left|\frac{1}{n_h^t(s,a)} \sum_{k=1}^t \chi_h^k(s,a) (p_h^k-p_h)\uV_{h+1}^{k-1}(s,a)\right|\\
    &\quad+\left|\frac{1}{n_h^t(s,a)}\sum_{k=1}^t \chi_h^k(s,a) \rgamma_h^k(s,a) (p_h^k-p_h)(\uV_{h+1}^{k-1}-V_{h,s,a}^{k-1})(s,a)\right|\,.
\end{align*}
We will upper-bound the two terms of the right-hand of the previous inequality separately. For the first term, thanks to the definition of $\cE$ (see Section~\ref{app:concentration_value}), we obtain
\[\frac{1}{n_h^t(s,a)} \left|\sum_{k=1}^t \chi_h^k(s,a) (p_h^k-p_h)\uV_{h+1}^{k-1}(s,a)\right| \leq \sqrt{\frac{2}{n_h^t(s,a)}\sum_{k=1}^t  \chi_h^k(s,a) \Var_{p_h}(\uV_h^{k-1})(s,a) \frac{\zeta}{n_h^t(s,a)}}+6H\frac{\zeta}{n_h^t(s,a)}\,.\]
For the second term using Lemma~\ref{lem:concentration_without_variance} yields
\begin{align*}
  &\left|\frac{1}{n_h^t(s,a)}\sum_{k=1}^t \chi_h^k(s,a) \rgamma_h^k(s,a) (p_h^k-p_h)(\uV_{h+1}^{k-1}-V_{h,s,a}^{k-1})(s,a)\right|\leq\\
  &\qquad \frac{1}{4H\log(T) n_h^t(s,a)}\sum_{k=1}^t  \chi_h^k(s,a) \rgamma_h^k(s,a) p_h(V_{h,s,a}^{k-1}-\uV_{h+1}^{k-1})(s,a)+ 14 H^3 \frac{\zeta\log(T)}{n_h^t(s,a)}\,.
\end{align*}
Combining these two inequalities allows us to conclude.
\end{proof}


The exploration bonus is designed to compensate the approximation error made by $Q_h^t$ in the previous lemma, as we show below. 

\begin{lemma}
\label{lem:bonus_lower_bound}
  On the event $\cE$, $\forall t \in \N, \forall h \in [H], \forall (s,a)\in\cS\times\cA$, if $n_h^t(s,a)>0$, it holds
\begin{align*}
  \beta_h^t(s,a)&\geq \sqrt{\frac{2}{n_h^t(s,a)} \sum_{k=1}^t  \chi_h^k(s,a) \Var_{p_h}(\uV_{h+1}^{k-1})(s,a)\frac{\zeta}{n_h^t(s,a)}} + 20H^3\frac{\zeta\log(T)}{n_h^t(s,a)}\\
  &\quad+ \frac{1}{4  H\log(T)  n_h^t(s,a)}\sum_{k=1}^t  \chi_h^k(s,a) \rgamma_h^k(s,a) p_h(V_{h,s,a}^{k-1}-\uV_{h+1}^{k-1})(s,a)\,.
\end{align*}
\end{lemma}
\begin{proof}
First, we recall the definition of the bonus
\[
\beta_h^t(s,a) = 2\sqrt{W_h^t(s,a)\frac{\zeta}{n_h^t(s,a)}}+53 H^3\frac{\zeta\log(T)}{n_h^t(s,a)}+ \frac{1}{H \log(T) n_h^t(s,a)}\sum_{k=1}^t  \chi_h^k(s,a)\rgamma_h^k(s,a) p_h^k(V_{h,s,a}^{k-1}-\uV_{h+1}^{k-1})(s,a)
\]
where $W_h^t$ is a proxy for the variance term
\[
W_h^t(s,a) = \frac{1}{n_h^t(s,a)} \sum_{k=1}^t \chi_h^k(s,a) p_h^k (\uV_{h+1}^{k-1})^2(s,a) -  \left(\frac{1}{n_h^t(s,a)}  \sum_{k=1}^t  \chi_h^k(s,a) p_h^k \uV_{h+1}^{k-1}(s,a) \right)^2\,.
\]

The approximation error in Lemma \ref{lem:concentration_Q} includes terms depending on the true transitions $p_h$, which are unknown to the algorithm. Hence, to design the bonuses, we will use the concentration inequalities that hold on the event $\cE$ to replace $p_h$ by $p_h^k$, which depends only on the observed data and can be used in the bonus.

\paragraph{Correction term} First note that thanks to Lemma~\ref{lem:concentration_without_variance} we can control the correction term
\begin{align}
  \left|\frac{1}{n_h^t(s,a)}\sum_{k=1}^t \chi_h^k(s,a) \rgamma_h^k(s,a) (p_h^k-p_h)(\uV_{h+1}^{k-1}-V_{h,s,a}^{k-1})(s,a)\right| &\leq 14 H^3\frac{\log(T)\zeta}{n_h^t(s,a)} \nonumber\\
  +\frac{1}{4H\log(T)n_h^t(s,a)}\sum_{k=1}^t  \chi_h^k(s,a)  \rgamma_h^k(s,a)& p_h(V_{h,s,a}^{k-1}-\uV_{h+1}^{k-1})(s,a)\label{eq:control_correction_bonus}\,.
\end{align}

\paragraph{Variance term $W_h^t(s,a)$} 
 Using Lemma~\ref{lem:concentration_without_variance} and the definition of $\cE$ (see Section~\ref{app:concentration_value}), we replace the sample "expectation" by the true expectation in the two sums of the proxy of the variance $W_h^t(s,a)$:
\begin{align}
   \label{eq:control_sum_square_W}
  \left|\frac{1}{n_h^t(s,a)}\sum_{k=1}^t \chi_h^k(s,a) (p_h^k-p_h)(\uV_{h+1}^{k-1})^2(s,a)\right| &\leq  \frac{1}{4}\frac{1}{n_h^t(s,a)}\sum_{k=1}^t  \chi_h^k(s,a) \Var_{p_h}(\uV_{h+1}^{k-1})(s,a)+44 H^2 \frac{\zeta}{n_h^t(s,a)}\,,
\end{align}
and
\begin{align}
  \Bigg| \Bigg(\frac{1}{n_h^t(s,a)}  \sum_{k=1}^t  \chi_h^k(s,a) &p_h^k \uV_{h+1}^{k-1}(s,a) \Bigg)^2 - \left(\frac{1}{n_h^t(s,a)}  \sum_{k=1}^t  \chi_h^k(s,a) p_h \uV_{h+1}^{k-1}(s,a) \right)^2 \Bigg|\nonumber\\
   &\leq  \frac{2H}{n_h^t(s,a)} \Bigg|  \sum_{k=1}^t  \chi_h^k(s,a) (p_h^k-p_h) \uV_{h+1}^{k-1}(s,a) \Bigg|   \nonumber\\
   &\leq H\sqrt{8 \frac{1}{n_h^t(s,a)} \sum_{k=1}^t  \chi_h^k(s,a) \Var_{p_h}(\uV_{h+1}^{k-1})(s,a)\frac{\zeta}{n_h^t(s,a)}}+12H^2 \frac{\zeta}{n_h^t(s,a)} \nonumber\\
   &\leq \frac{1}{4n_h^t(s,a)} \sum_{k=1}^t  \chi_h^k(s,a) \Var_{p_h}(\uV_{h+1}^{k-1})(s,a)+44H^2 \frac{\zeta}{n_h^t(s,a)} \label{eq:control_square_sum_W}\,,
\end{align}
where we also used the fact that $\sqrt{xy} \leq x+y$.

 Using \eqref{eq:control_sum_square_W}, \eqref{eq:control_square_sum_W} and Jensen's inequality, we lower-bound $W_h^t(s, a)$:
 \begin{align*}
 	W_h^t(s,a) 
 	& = \frac{1}{n_h^t(s,a)} \sum_{k=1}^t \chi_h^k(s,a) p_h^k (\uV_{h+1}^{k-1})^2(s,a) -  \left(\frac{1}{n_h^t(s,a)}  \sum_{k=1}^t  \chi_h^k(s,a) p_h^k \uV_{h+1}^{k-1}(s,a) \right)^2
 	\\
 	& 
 	\geq 
 	\frac{1}{n_h^t(s,a)} \sum_{k=1}^t \chi_h^k(s,a) p_h (\uV_{h+1}^{k-1})^2(s,a) -  \left(\frac{1}{n_h^t(s,a)}  \sum_{k=1}^t  \chi_h^k(s,a) p_h \uV_{h+1}^{k-1}(s,a) \right)^2
 	\\
 	& 
 	\quad 
 	- \frac{1}{2}\frac{1}{n_h^t(s,a)}\sum_{k=1}^t  \chi_h^k(s,a) \Var_{p_h}(\uV_{h+1}^{k-1})(s,a) - 88 H^2 \frac{\zeta}{n_h^t(s,a)}
 	\quad \text{by \eqref{eq:control_sum_square_W} and \eqref{eq:control_square_sum_W}}
 	\\
 	&
 		\geq 
 	\frac{1}{n_h^t(s,a)} \sum_{k=1}^t \chi_h^k(s,a) p_h (\uV_{h+1}^{k-1})^2(s,a) -  \frac{1}{n_h^t(s,a)}  \sum_{k=1}^t  \chi_h^k(s,a) (p_h \uV_{h+1}^{k-1}(s,a))^2 
 	\\
 	& 
 	\quad 
 	- \frac{1}{2}\frac{1}{n_h^t(s,a)}\sum_{k=1}^t  \chi_h^k(s,a) \Var_{p_h}(\uV_{h+1}^{k-1})(s,a) - 88 H^2 \frac{\zeta}{n_h^t(s,a)}
 	\quad \text{by Jensen's inequality}
 	\\
 	& \geq \frac{1}{2n_h^t(s,a)} \sum_{k=1}^t  \chi_h^k(s,a) \Var_{p_h}(\uV_{h+1}^{k-1})(s,a)-88H^2 \frac{\zeta}{n_h^t(s,a)}\,.
 \end{align*}


Finally, combining the inequality above with \eqref{eq:control_correction_bonus} for the correction term allows us to conclude
\begin{align*}
  \beta_h^t(s,a) &\geq 2\sqrt{ \left(W_h^t(s,a) + 88H^2 \frac{\zeta}{n_h^t(s,a)}\right) \frac{\zeta}{n_h^t(s,a)}} + (53-2\sqrt{88}-14)H^3\frac{\zeta\log(T)}{n_h^t(s,a)}\\
  &\quad+ \frac{3}{4 H \log(T) n_h^t(s,a)}\sum_{k=1}^t  \chi_h^k(s,a)  \rgamma_h^k(s,a) p_h(V_{h,s,a}^{k-1}-\uV_{h+1}^{k-1})(s,a)\\
  &\geq \sqrt{\frac{2}{n_h^t(s,a)} \sum_{k=1}^t  \chi_h^k(s,a) \Var_{p_h}(\uV_{h+1}^{k-1})(s,a)\frac{\zeta}{n_h^t(s,a)}} + 20 H^3\frac{\zeta\log(T)}{n_h^t(s,a)}\\
  &\quad+ \frac{1}{4 H \log(T) n_h^t(s,a)}\sum_{k=1}^t  \chi_h^k(s,a)  \rgamma_h^k(s,a) p_h(V_{h,s,a}^{k-1}-\uV_{h+1}^{k-1})(s,a)\,,
\end{align*}
where we used the fact that $\sqrt{x+y} \leq \sqrt{x}+\sqrt{y}$.

\end{proof}
We are now ready to prove the optimism.
\lemoptimism*

\begin{proof}
We proceed by induction on $t$. For $t= 0$ the result is trivially true because of the initialization. Assume the result is true for all $k \leq t-1$. We will prove the results at episode $t$ by backward induction on $h$. For $h=H+1$ the result is trivially true because $\uV_{H+1}^t(s)=\Vstar_{H+1}(s)=0$. Assume the results are true at $h+1$. If $n_h^t(s,a)=0$ because $\uQ_h^t(s,a) = H$ in this case we have
$ \uQ_h^t(s,a) \geq \Qstar_h(s,a)$. If $n_h^t(s,a)>0$, since the event $\cE$ holds, Lemma~\ref{lem:concentration_Q} and Lemma~\ref{lem:bonus_lower_bound} yield
\begin{align*}
  \uQ_h^t(s,a) &= Q_h^t(s,a)+\beta_h^t(s,a)\\
  &\geq r_h(s,a) + p_h V_{h,s,a}^t(s,a)\\
  &\geq r_h(s,a) + p_h\!\left(\sum_{k=1}^{t} \teta_h^{t,k}(s,a) \uV_{h+1}^{k-1} \right)(s,a) \geq r_h(s,a) + p_h \Vstar_{h+1}(s,a) = \Qstar_h(s,a)
\end{align*}
where in the last inequality we used the induction assumption. To conclude it remains to note that
\[
\uV_h^t(s) = \clip\!\big(\max_a\uQ_h^t(s,a),0,H\big) \geq \max_{a\in\cA} \Qstar_h(s,a)  =  \Vstar_{h+1}(s).
\]
\end{proof}

In particular, on the event $\cE$ we have the ordering for all $(s,a,h,s')$ and $t\in[T]$:
\[
V_h^{\pi^t}(s') \leq \Vstar_h(s') \leq \uV_h^{t-1}(s) \leq  V_{h,s,a}^{t-1}(s)\,.
\]

 \newpage
\section{Proof of the regret bound}
\label{app:proof_regret_bound}
We introduce the maximum between the count and one to deal with the state-action never visited:
\[
\tn_h^t(s,a) = \max(n_h^t(s,a),1)\,.
\]
We provide a refined version of Lemma~\ref{lem:concentration_Q} where we introduce the variance of the value function of the current policy rather than the variance of the upper-bound in order to apply subsequently the law of total variance (Lemma~\ref{lem:law_of_total_variance}).
\begin{lemma}
  \label{lem:concentration_Q_regret}
  On the event $\cE$, $\forall t \in \N, \forall h \in [H], \forall (s,a)\in\cS\times\cA$, it holds
  \begin{align*}
    \left|Q_h^t(s,a)- r_h(s,a) - p_h V_{h,s,a}^t(s,a)\right| &\leq  \sqrt{\frac{4}{\tn_h^t(s,a)}\sum_{k=1}^t  \chi_h^k(s,a) \Var_{p_h}(V_{h+1}^{\pi^k})(s,a) \frac{\zeta}{\tn_h^t(s,a)}}\\
    &\quad+ \frac{2}{H\log(T)\tn_h^t(s,a)}\sum_{k=1}^t  \chi_h^k(s,a) p_h(\uV_{h+1}^{k-1}-V_{h+1}^{\pi^k})(s,a)\\
    &\quad+ 24 H^3\frac{\log(T)\zeta}{\tn_h^t(s,a)}\,.
  \end{align*}
\end{lemma}
\begin{proof}
If $n_h^t(s,a) = 0$ the bound is trivially true because in this case $\left|Q_h^t(s,a)- r_h(s,a) - p_h V_{h,s,a}^t(s,a)\right|=\left| r_h(s,a) + p_h V_{h,s,a}^t(s,a)\right|\leq (H+1)$. Now assume that $n_h^t(s,a)>0$.
Proceeding as in the proof of Lemma~\ref{lem:concentration_Q} we have on the event $\cE$
\begin{align*}
  \left|Q_h^t(s,a)- r_h(s,a) - p_h V_{h,s,a}^t(s,a)\right| &\leq \sqrt{\frac{2}{n_h^t(s,a)}\sum_{k=1}^t  \chi_h^k(s,a) \Var_{p_h}(\uV_{h+1}^{k-1})(s,a) \frac{\zeta}{n_h^t(s,a)}}+ 20 H^3 \frac{\zeta\log(T)}{n_h^t(s,a)}\\
  &+\frac{1}{4\log(T) H n_h^t(s,a)}\sum_{k=1}^t  \chi_h^k(s,a) \rgamma_h^k(s,a) p_h(V_{h,s,a}^{k-1}-\uV_{h+1}^{k-1})(s,a)\,.
\end{align*}

\paragraph{Correction term} Using~\eqref{eq:V_expand} and $V_{h,s,a}^{t}\geq \Vstar_{h+1} \geq V_{h+1}^{\pi^k}$, we upper-bound the correction term
\begin{align}
  \frac{1}{n_h^t(s,a)}\sum_{k=1}^t  \chi_h^k(s,a)  \rgamma_h^k(s,a) p_h(V_{h,s,a}^{k-1}-\uV_{h+1}^{k-1})(s,a) &=   \frac{1}{n_h^t(s,a)}\sum_{k=1}^t  \chi_h^k(s,a) p_h\uV_{h+1}^{k-1}(s,a) - p_h V_{h,s,a}^t(s,a)\nonumber\\
  &\leq   \frac{1}{n_h^t(s,a)}\sum_{k=1}^t  \chi_h^k(s,a) p_h(\uV_{h+1}^{k-1}-V_{h+1}^{\pi^k})(s,a)\label{eq:ub_correct_lem_regret}\,.
\end{align}

\paragraph{Variance term} For the variance term, using $H \geq \uV_{h}^{k} \geq \Vstar_{h} \geq V_{h}^{\pi^{k+1}}$ and Lemma~\ref{lem:switch_variance}, we can replace the variance of the current upper bounds on the optimal value function by the variance of the current policy,
\begin{align*}
  \frac{1}{n_h^t(s,a)}\sum_{k=1}^t  \chi_h^k(s,a) \Var_{p_h}(\uV_{h+1}^{k-1})(s,a)&\leq \frac{2}{n_h^t(s,a)}\sum_{k=1}^t  \chi_h^k(s,a) \Var_{p_h}(V_{h+1}^{\pi^k})(s,a)\\
  &\quad+ \frac{2H}{n_h^t(s,a)}\sum_{k=1}^t  \chi_h^k(s,a) p_h(\uV_{h+1}^{k-1}-V_{h+1}^{\pi^k})(s,a)\,.
\end{align*}
Using $\sqrt{x+y} \leq \sqrt{x}+\sqrt{y}$ and $\sqrt{xy}\leq x+y$ allows to upper-bound the variance term
\begin{align}
  \sqrt{\frac{2}{n_h^t(s,a)}\sum_{k=1}^t  \chi_h^k(s,a) \Var_{p_h}( \uV_{h+1}^{k-1})(s,a) \frac{\zeta}{n_h^t(s,a)}} &\leq   \sqrt{\frac{4}{n_h^t(s,a)}\sum_{k=1}^t  \chi_h^k(s,a) \Var_{p_h}(V_{h+1}^{\pi^k})(s,a) \frac{\zeta}{n_h^t(s,a)}}\nonumber\\
  &\quad+ \frac{1}{H\log(T)n_h^t(s,a)}\sum_{k=1}^t  \chi_h^k(s,a) p_h(\uV_{h+1}^{k-1}-V_{h+1}^{\pi^k})(s,a) \nonumber\\
  &\quad+ 4 H^2\frac{\log(T)\zeta}{n_h^t(s,a)}\label{eq:ub_var_lem_regret}\,.
\end{align}
Combining \eqref{eq:ub_var_lem_regret} and \eqref{eq:ub_correct_lem_regret} allows us to conclude.
\end{proof}

We now provide an upper bound on the bonus.
\begin{lemma}
\label{lem:bonus_upper_bound}
  On the event $\cE$, $\forall t \in \N, \forall h \in [H], \forall (s,a)\in\cS\times\cA$, it holds
\begin{align*}
  \beta_h^t(s,a)&\leq   2\sqrt{\frac{3}{\tn_h^t(s,a)} \sum_{k=1}^t  \chi_h^k(s,a) \Var_{p_h}(V_{h+1}^{\pi^k})(s,a)\frac{\zeta}{\tn_h^t(s,a)}} \\
  &\quad+\frac{3}{H\log(T) \tn_h^t(s,a)}  \sum_{k=1}^t  \chi_h^k(s,a) p_h(\uV_{h+1}^{k-1} - V_{h+1}^{\pi^k})(s,a) +  106 H^3 \frac{\log(T)\zeta}{\tn_h^t(s,a)}\,.
\end{align*}
\end{lemma}
\begin{proof}
  If $n_h^t(s,a)=0$ the result is trivially true because in this case $\beta_h^t(s,a) =H$. We now assume $n_h^t(s,a)>0$. For the upper-bound we first upper-bound the proxy of the variance. Using \eqref{eq:control_sum_square_W} and \eqref{eq:control_square_sum_W} from the proof of Lemma~\ref{lem:bonus_lower_bound} and the fact that $H\geq \uV_{h}^k \geq \Vstar_h$ in combination with Lemma~\ref{lem:optimism} (optimism) we obtain
  \begin{align*}
    W_h^t(s,a) 
    &
    \leq \frac{1}{n_h^t(s,a)}\sum_{k=1}^t \chi_h^k(s,a) p_h(\uV_{h+1}^{k-1})^2(s,a) -  \left(\frac{1}{n_h^t(s,a)}  \sum_{k=1}^t  \chi_h^k(s,a) p_h \uV_{h+1}^{k-1}(s,a) \right)^2
    \\
    &
    \quad+\frac{1}{2n_h^t(s,a)} \sum_{k=1}^t  \chi_h^k(s,a) \Var_{p_h}(\uV_{h+1}^{k-1})(s,a)+88H^2 \frac{\zeta}{n_h^t(s,a)}
    \\
    &
    = \Var_{p_h}(\Vstar_{h+1})(s,a)\\
    &\quad +\frac{1}{n_h^t(s,a)}\sum_{k=1}^t \chi_h^k(s,a) p_h\big((\uV_{h+1}^{k-1})^2 - (\Vstar_{h+1})^2 \big) (s,a)
    + (p_h \Vstar_{h+1})^2-  \left(\frac{1}{n_h^t(s,a)} \sum_{k=1}^t  \chi_h^k(s,a) p_h \uV_{h+1}^{k-1}(s,a) \right)^2\\
    &\quad+\frac{1}{2n_h^t(s,a)} \sum_{k=1}^t  \chi_h^k(s,a) \Var_{p_h}(\uV_{h+1}^{k-1})(s,a)+88H^2 \frac{\zeta}{n_h^t(s,a)}\\
    &\leq \Var_{p_h}(\Vstar_{h+1})(s,a) + \frac{1}{2n_h^t(s,a)} \sum_{k=1}^t  \chi_h^k(s,a) \Var_{p_h}(\uV_{h+1}^{k-1})(s,a)\\
    &\quad +\frac{2H}{n_h^t(s,a)}  \sum_{k=1}^t  \chi_h^k(s,a) p_h(\uV_{h+1}^{k-1} - \Vstar_{h+1})(s,a) +88H^2 \frac{\zeta}{n_h^t(s,a)}\,.
  \end{align*}
  We now introduce the value of the current policy and proceed similarly as above using $H \geq \uV_{h}^{k} \geq \Vstar_{h} \geq V_{h}^{\pi^{k+1}}$. Also, we apply Lemma~\ref{lem:switch_variance} to the terms $\Var_{p_h}(\Vstar_{h+1})(s,a)$ and $\Var_{p_h}(\uV_{h+1}^{k-1})(s,a)$ to make appear $ \Var_{p_h}(V_{h+1}^{\pi^k})$:
  \begin{align*}
      W_h^t(s,a) &\leq  \frac{3}{n_h^t(s,a)} \sum_{k=1}^t  \chi_h^k(s,a) \Var_{p_h}(V_{h+1}^{\pi^k})(s,a)\\
      &\quad +\frac{2H}{n_h^t(s,a)}  \sum_{k=1}^t  \chi_h^k(s,a) p_h(\Vstar_{h+1} - V_{h+1}^{\pi^k})(s,a) + \frac{H}{n_h^t(s,a)}  \sum_{k=1}^t  \chi_h^k(s,a) p_h(\uV_{h+1}^{k-1}-V_{h+1}^{\pi^k})(s,a) \\
      &\quad +\frac{2H}{n_h^t(s,a)}  \sum_{k=1}^t  \chi_h^k(s,a) p_h(\uV_{h+1}^{k-1} - \Vstar_{h+1})(s,a) +88H^2 \frac{\zeta}{n_h^t(s,a)}\\
      &\leq  \frac{3}{n_h^t(s,a)} \sum_{k=1}^t  \chi_h^k(s,a) \Var_{p_h}(V_{h+1}^{\pi^k})(s,a)\\
      &\quad +\frac{5H}{n_h^t(s,a)}  \sum_{k=1}^t  \chi_h^k(s,a) p_h(\uV_{h+1}^{k-1} - V_{h+1}^{\pi^k})(s,a) +88H^2 \frac{\zeta}{n_h^t(s,a)}
  \end{align*}

Combining this inequality with $\sqrt{x+y}\leq \sqrt{x}+\sqrt{y}$ and $\sqrt{xy}\leq x+y$ we upper-bound the variance term of the bonus
\begin{align}
 2\sqrt{W_h^t(s,a)\frac{\zeta}{n_h^t(s,a)}} &\leq  2\sqrt{\frac{3}{n_h^t(s,a)} \sum_{k=1}^t  \chi_h^k(s,a) \Var_{p_h}(V_{h+1}^{\pi^k})(s,a)\frac{\zeta}{n_h^t(s,a)}}\nonumber\\
 &\quad+ \frac{1}{H\log(T) n_h^t(s,a)}  \sum_{k=1}^t  \chi_h^k(s,a) p_h(\uV_{h+1}^{k-1} - V_{h+1}^{\pi^k})(s,a) + 39 H^2 \frac{\log(T)\zeta}{n_h^t(s,a)} \label{eq:ub_variance_bonus}\,.
\end{align}
We can proceed similarly for the correction term, using Lemma~\ref{lem:concentration_without_variance}
\begin{align*}
  \frac{1}{  H\log(T)  n_h^t(s,a)}\sum_{k=1}^t  \chi_h^k(s,a) \rgamma_h^k(s,a) p_h^k(V_{h,s,a}^{k-1}-\uV_{h+1}^{k-1})(s,a) &\leq 14 H^3 \frac{\log(T)\zeta}{n_h^t(s,a)}\nonumber\\
  \quad + \frac{5}{4H\log(T)n_h^t(s,a)}\sum_{k=1}^t  \chi_h^k(s,a)&  \rgamma_h^k(s,a) p_h(V_{h,s,a}^{k-1}-\uV_{h+1}^{k-1})(s,a) \\
  &\leq 14 H^3 \frac{\log(T)\zeta}{n_h^t(s,a)}\nonumber\\
  \quad + \frac{5}{4H\log(T)n_h^t(s,a)}\sum_{k=1}^t  \chi_h^k(s,a)&  p_h(\uV_{h+1}^{k-1}-V_{h+1}^{\pi^k})(s,a) \,,
\end{align*}
where in the last inequality we use, thanks to~\eqref{eq:V_expand} and $V_{h+1}^{\pi^k}(s') \leq V_{h,s,a}^t(s')$,
\begin{align*}
  \frac{1}{n_h^t(s,a)}\sum_{k=1}^t  \chi_h^k(s,a)  \rgamma_h^k(s,a) (V_{h,s,a}^{k-1}-\uV_{h+1}^{k-1})(s') &=   \frac{1}{n_h^t(s,a)}\sum_{k=1}^t  \chi_h^k(s,a) \uV_{h+1}^{k-1}(s') - V_{h,s,a}^t(s')\\
  &\leq   \frac{1}{n_h^t(s,a)}\sum_{k=1}^t  \chi_h^k(s,a) (\uV_{h+1}^{k-1}-V_{h+1}^{\pi^k})(s')\,.
\end{align*}
Combining these two bounds allows us to conclude
\end{proof}
We are now ready to prove the main result.
\begin{proof}[Proof of Theorem~\ref{th:regret_UCBMQ}]
We will prove that the regret bound holds on event $\cD$ (Section~\ref{app:concentration_master}). Not his events holds with probability at least $1-\delta$, according to Lemma~\ref{lem:master_event}. Thus from now we assume that the event $\cD$ holds. Fix $(s,a,h) \in\cS\times\cA\times[H]$ and $t\geq 0$.

\paragraph{Step 1: Upper-bound $(\uQ_h^{t}-Q_h^{\pi^{t+1}})(s,a)$}

We will upper-bound the difference of the previous upper bound on the optimal Q-value function and the Q-value of the current policy. Thanks to Lemma~\ref{lem:concentration_Q_regret} and Lemma~\ref{lem:bonus_upper_bound}, we obtain
\begin{align}
(\uQ_h^{t}-Q_h^{\pi^{t+1}})(s,a) &\leq p_h(V_{h,s,a}^t-V_{h+1}^{\pi^{t+1}})(s,a)+ 130 H^3 \frac{\log(T)\zeta}{\tn_h^t(s,a)} \nonumber\\
&\quad+ \sqrt{\frac{30}{\tn_h^t(s,a)} \sum_{k=1}^t  \chi_h^k(s,a) \Var_{p_h}(V_{h+1}^{\pi^k})(s,a)\frac{\zeta}{\tn_h^t(s,a)}} \nonumber\\
&\quad +\frac{5}{H\log(T) \tn_h^t(s,a)}  \sum_{k=1}^t  \chi_h^k(s,a) p_h(\uV_{h+1}^{k-1} - V_{h+1}^{\pi^k})(s,a)\,.\label{eq:ub_dif_Q_regret}
\end{align}

\paragraph{Step 2: Upper-bound of the local optimistic regret $\tR_h^T(s,a)$}
We will now thanks to this inequality upper-bound the local optimistic regret $\tR_h^T(s,a)$  at state-action $s,a$ and step $h$ defined by
\[ \tR_h^T(s,a) \triangleq \sum_{t=0}^{T-1} \chi_h^{t+1}(s,a) (\uQ_h^{t}-Q_h^{\pi^{t+1}})(s,a)\,.
\]
We will upper-bound the sum over $t$ weighted by the indicator function that the state-action $(s,a)$ is visited at time $t+1$ of each term in the above inequality \eqref{eq:ub_dif_Q_regret}. For the first term, we introduce the optimal value function
\[
 p_h(V_{h,s,a}^t-V_{h+1}^{\pi^{t+1}})(s,a) = p_h(V_{h,s,a}^t-\Vstar_{h+1})(s,a)+p_h(\Vstar_{h+1}-V_{h+1}^{\pi^{t+1}})(s,a)\,.
\]
Then using~\eqref{eq:V_with_weights} and Lemma~\ref{lem:properties_weights} yields
\begin{align*}
  \sum_{t=0}^{T-1}  \chi_h^{t+1}(s,a) p_h(V_{h,s,a}^t-\Vstar_{h+1})(s,a) &= \sum_{t=0}^{T-1}  \chi_h^{t+1}(s,a) \ind_{\{n_h^t(s,a)=0\}} p_h(V_{h,s,a}^t-\Vstar_{h+1})(s,a)\\
  &+ \sum_{t=0}^{T-1}  \chi_h^{t+1}(s,a) \ind_{\{n_h^t(s,a)>0\}}\sum_{k=1}^t \teta_h^{t,k}(s,a) p_h(\uV_{h+1}^{k-1}-\Vstar_{h+1})(s,a)\\
  &\leq H + \sum_{k=1}^{T-1}  \left(\sum_{t=k}^{T-1}  \chi_h^{t+1}(s,a) \teta_h^{t,k}(s,a) \right) p_h(\uV_{h+1}^{k-1}-\Vstar_{h+1})(s,a)\\
  &\leq H + \left(1+\frac{1}{H}\right)\sum_{t=0}^{T-1}\chi_h^{t+1}(s,a) p_h(\uV_{h+1}^{t-1}-\Vstar_{h+1})(s,a)\,.
\end{align*}
Combining this inequality with the previous decomposition and that $\Vstar_{h+1}\geq V_{h+1}^{\pi^{k+1}}$, one obtains
\begin{align}
  \sum_{t=0}^{T-1}  \chi_h^{t+1}(s,a) p_h(V_{h,s,a}^t-V_{h+1}^{\pi^{t+1}})(s,a) &\leq \sum_{t=0}^{T-1}  \chi_h^{t+1}(s,a) p_h(\Vstar_{h+1}-V_{h+1}^{\pi^{t+1}})(s,a) +H \nonumber\\
  &\quad + \left(1+\frac{1}{H}\right)\sum_{t=0}^{T-1}\chi_h^{t+1}(s,a) p_h(\uV_{h+1}^{t-1}-\Vstar_{h+1})(s,a)\nonumber\\
  &\leq H+\left(1+\frac{1}{H}\right)\sum_{t=0}^{T-1}\chi_h^{t+1}(s,a) p_h(\uV_{h+1}^{t-1}-V_{h+1}^{\pi^{t+1}})(s,a)\,.\label{eq:ub_main_term_regret}
\end{align}
We can proceed in a similar way but using this time Lemma~\ref{lem:sum_1_over_n_history} to upper-bound the sum of the correction terms, precisely
\begin{align*}
  \sum_{t=0}^{T-1} \frac{\chi_h^{t+1}(s,a)}{\tn_h^t(s,a)}  \sum_{k=1}^t  \chi_h^k(s,a) p_h(\uV_{h+1}^{k-1} - V_{h+1}^{\pi^k})(s,a) &\leq \sum_{k=1}^{T-1} \left(\sum_{t=k}^{T-1}\frac{\chi_h^{t+1}(s,a)}{\tn_h^t(s,a)} \right) \chi_h^k(s,a) p_h(\uV_{h+1}^{k-1} - V_{h+1}^{\pi^k})(s,a) \\
  &\leq  8\log(T)\sum_{k=1}^{T-1} \chi_h^{k}(s,a) p_h(\uV_{h+1}^{k-1} - V_{h+1}^{\pi^{k}})(s,a)\,.
\end{align*}
We then obtain the upper bound on the correction term
\begin{align}
  \sum_{t=0}^{T-1} \frac{5\chi_h^{t+1}(s,a)}{H\log(T) \tn_h^t(s,a)}  \sum_{k=1}^t  \chi_h^k(s,a) p_h(\uV_{h+1}^{k-1} - V_{h+1}^{\pi^k})(s,a) &\leq   \frac{40}{H}  \sum_{t=0}^{T-1} \chi_h^{t+1}(s,a) p_h(\uV_{h+1}^{t} - V_{h+1}^{\pi^{t+1}})(s,a)\,.
  \label{eq:ub_correct_regret}
\end{align}
For the variance term using Cauchy-Schwarz inequality in combination with Lemma~\ref{lem:sum_1_over_n_history} and Lemma~\ref{lem:sum_1_over_n} yields
\begin{align}
  \sum_{t=0}^{T-1} \chi_h^{t+1}(s,a) \sqrt{\frac{30}{\tn_h^t(s,a)} \sum_{k=1}^t  \chi_h^k(s,a) \Var_{p_h}(V_{h+1}^{\pi^k})(s,a)\frac{\zeta}{\tn_h^t(s,a)}}  &\leq
  \sqrt{30   \sum_{t=0}^{T-1}\frac{ \chi_h^{t+1}(s,a)}{\tn_h^t(s,a)}   \sum_{k=1}^t \Var_{p_h}(V_{h+1}^{\pi^k})(s,a) }\nonumber\\
  &\times \sqrt{  \sum_{t=0}^{T-1} \chi_h^{t+1}(s,a) \frac{\zeta}{\tn_h^t(s,a)}}\nonumber\\
  \leq  44 &\log(T)\zeta^{1/2} \sqrt{\sum_{t=0}^{T-1} \chi_h^{t+1}(s,a) \Var_{p_h}(V_{h+1}^{\pi^{t+1}})(s,a)}\,. \label{eq:ub_var_regret}
\end{align}
Finally for the remaining term, using Lemma~\ref{lem:sum_1_over_n}, we have
\begin{align}
  \sum_{t=0}^{T-1} \chi_h^{t+1}(s,a) 130 H^3 \frac{\log(T)\zeta}{\tn_h^t(s,a)} \leq 1040 H^3 \log(T)^2\zeta\,.
  \label{eq:ub_bonus_regret}
\end{align}
Thus combining from \eqref{eq:ub_main_term_regret} to \eqref{eq:ub_bonus_regret} with \eqref{eq:ub_dif_Q_regret}
 we obtain an upper-bound on the optimistic regret at $(s,a,h)$
\begin{align}
  \tR_h^T(s,a) &\leq 44 \log(T)\zeta^{1/2} \sqrt{\sum_{t=0}^{T-1} \chi_h^{t+1}(s,a) \Var_{p_h}(V_{h+1}^{\pi^{t+1}})(s,a)} +  1041 H^3 \log(T)^2\zeta \nonumber\\
  &\quad+\left(1+\frac{41}{H}\right)\sum_{t=0}^{T-1}\chi_h^{t+1}(s,a) p_h(\uV_{h+1}^{t}-V_{h+1}^{\pi^{t+1}})(s,a)\,. \label{eq:ub_local_regret}
\end{align}

\paragraph{Step 3: From visit $\chi_h^t$ to reach probability $\bp_h^t$} We replace the indicator function $\chi_h^t$ by its expectation $\bp_h^t$. Since we are on the event $\cD$, in particular the event $\cG$ holds. Thus we know that
\begin{align*}
  \sqrt{\sum_{t=0}^{T-1} \chi_h^{t+1}(s,a) \Var_{p_h}(V_{h+1}^{\pi^{t+1}})(s,a)} \leq \sqrt{2\sum_{t=0}^{T-1} \bp_h^{t+1}(s,a) \Var_{p_h}(V_{h+1}^{\pi^{t+1}})(s,a)} + \sqrt{8\zeta} H
\end{align*}
and
\begin{align*}
  \sum_{t=0}^{T-1}\chi_h^{t+1}(s,a) p_h(\uV_{h+1}^{t}-V_{h+1}^{\pi^{t+1}})(s,a) \leq \left(1+\frac{1}{H}\right) \sum_{t=0}^{T-1}p_h^{t+1}(s,a) p_h(\uV_{h+1}^{t}-V_{h+1}^{\pi^{t+1}})(s,a) + 14H^2\zeta
\end{align*}
Plugging these two inequalities in \eqref{eq:ub_local_regret} we obtain
\begin{align}
  \tR_h^T(s,a) &\leq 63 \log(T)\zeta^{1/2} \sqrt{\sum_{t=0}^{T-1} \bp_h^{t+1}(s,a) \Var_{p_h}(V_{h+1}^{\pi^{t+1}})(s,a)} +  1754 H^3 \log(T)^2\zeta \nonumber\\
  &\quad+\left(1+\frac{83}{H}\right)\sum_{t=0}^{T-1} \bp_h^{t+1}(s,a) p_h(\uV_{h+1}^{t}-V_{h+1}^{\pi^{t+1}})(s,a)\label{eq:ub_local_regret_prb}\,.
\end{align}

\paragraph{Step 4: Upper-bound $\tR_h^T$ the step $h$ optimistic regret} We define the regret at step $h$ by
\[
\tR^T_h = \sum_{s\in\cS} \sum_{t=0}^{T-1} \bp_h^{t+1}(s)(\uV_h^{t-1}-V_h^{\pi^{t+1}})(s)\,.
\]
Note that in this definition we used the probability to reach state-action $(s,a)$ rather than the indicator function. We will upper bound this quantity with the local regret. Using successively, that the event $\cG$ holds (in particular event $\cG^{\mathrm{v}_2}$ see Appendix~\ref{app:concentration_count}), for all $x\geq 1$ it holds $1/\big(1-1/(4x)\big)  \leq 1+\frac{1}{x}$, the definition of $\uV_h^{k}(s)$ and that $\uQ_h^{k}\geq 0$ on $\cD$ (Lemma~\ref{lem:optimism}) we have
\begin{align*}
  \sum_{t=0}^{T-1} \bp_h^{t+1}(s) (\uV_h^{t}-V_h^{\pi^{t+1}})(s) &\leq  \frac{1}{1-1/(4H)} \sum_{t=0}^{T-1} \chi_h^{t+1}(s) (\uV_h^t(s)-V_h^{\pi^{t+1}})(s) + \frac{4}{3} 14 H^2\zeta\\
  &\leq  \left(1+\frac{1}{H}\right) \sum_{t=0}^{T-1} \chi_h^{t+1}(s) (\uV_h^t(s)-V_h^{\pi^{t+1}})(s) + 19 H^2 \zeta\\
  &\leq \left(1+\frac{1}{H}\right) \sum_{t=0}^{T-1} \chi_h^{t+1}(s)  \pi_{h}^{t+1}(\uQ_h^{k}-Q_h^{\pi^{t+1}})(s)+ 19 H^2 \zeta\,.
\end{align*}
Combining this inequality with~\eqref{eq:ub_local_regret_prb} then the fact the policies $\pi^t$ are deterministic and Cauchy-Schwarz inequality yield the upper-bound the step $h$ optimistic regret
\begin{align}
  \tR^T_h&\leq  \left(1+\frac{1}{H}\right) \sum_{s\in\cS} \sum_{t=0}^{T-1}   \chi_h^{t+1}(s) \pi_{h}^{t+1}(\uQ_h^{k}-Q_h^{\pi^{t+1}})(s) + 19 H^2 S \zeta\nonumber\\
  &=  \left(1+\frac{1}{H}\right) \sum_{(s,a)\in\cS\times\cA} \sum_{t=0}^{T-1}   \chi_h^{t+1}(s,a)(\uQ_h^{k}-Q_h^{\pi^{t+1}})(s,a) + 19 H^2 S \zeta\nonumber\\
  &=\left(1+\frac{1}{H}\right) \sum_{(s,a)\in\cS\times\cA} \tR_h^T(s,a) + 19 H^2 S \zeta\nonumber\\
  &\leq 126 \log(T)\zeta^{1/2} \sum_{(s,a)\in\cS\times\cA}\sqrt{\sum_{t=0}^{T-1} \bp_h^{t+1}(s,a) \Var_{p_h}(V_{h+1}^{\pi^{t+1}})(s,a)}  \nonumber\\
  &\quad+\left(1+\frac{167}{H}\right)\sum_{(s,a)\in\cS\times\cA} \sum_{t=0}^{T-1} \bp_h^{t+1}(s,a) p_h(\uV_{h+1}^{t}-V_{h+1}^{\pi^{t+1}})(s,a) +  3527 S A H^3 \log(T)^2\zeta\nonumber\\
  &\leq 126 \log(T) \sqrt{\zeta SA\sum_{(s,a)\in\cS\times\cA}\sum_{t=0}^{T-1} \bp_h^{t+1}(s,a) \Var_{p_h}(V_{h+1}^{\pi^{t+1}})(s,a)}  \nonumber\\
  &\quad+\left(1+\frac{167}{H}\right)\tR_{h+1}^T +  3527 H^3 S A \log(T)^2\zeta\,,
  \label{eq:ub_step_h_regret}
\end{align}
where in the last inequality we used that
\begin{align*}
\sum_{(s,a)\in\cS\times\cA}\bp_h^{t+1}(s,a) p_h(\uV_{h+1}^{t}-V_{h+1}^{\pi^{t+1}})(s,a) &=  \sum_{(s,a)\in\cS\times\cA}\sum_{s'\in\cS} \bp_h^{t+1}(s,a) p_h(s'|s,a)(\uV_{h+1}^{t}-V_{h+1}^{\pi^{t+1}})(s')\\
&= \sum_{s'\in\cS} \bp_{h+1}^{t+1}(s') (\uV_{h+1}^{t}-V_{h+1}^{\pi^{t+1}})(s')\,.
\end{align*}
\paragraph{Step 5: Upper-bound on the regret $R^T$}
We upper-bound the step $1$ regret $\tR_1$. By successively unfolding~\eqref{eq:ub_step_h_regret} with the fact that $\tR_{h+1}^T=0$, using the Cauchy-Schwarz inequality and the law of total variance (Lemma~\ref{lem:law_of_total_variance} in Appendix~\ref{app:Bellman_variance}), we obtain
\begin{align*}
  \tR_{1}^T &\leq \sum_{h=1}^H \left(1+\frac{127}{H}\right)^{H-h}  126 \log(T) \sqrt{\zeta SA\sum_{(s,a)\in\cS\times\cA}\sum_{t=0}^{T-1} \bp_h^{t+1}(s,a) \Var_{p_h}(V_{h+1}^{\pi^{t+1}})(s,a)}\\
  &\quad+ \left(1+\frac{127}{H}\right)^{H-h} 3527 H^3 S A  \log(T)^2\zeta\\
  &\leq  126  e^{127} \log(T) \sqrt{\zeta SAH\sum_{(s,a,h)\in\cS\times\cA\times[H]}\sum_{t=0}^{T-1} \bp_h^{t+1}(s,a) \Var_{p_h}(V_{h+1}^{\pi^{t+1}})(s,a)}+ 3527 e^{127} H^4 S A \log(T)^2\zeta\\
  &= 126  e^{127} \log(T) \sqrt{\zeta SAH\sum_{t=0}^{T-1} \E_{\pi^{t+1}}\left[\left(\sum_{h=1}^H r(s_h,a_h)  -V_1^{\pi^{t+1}}(s_1)\right)^2\right]}+ 3527 e^{127}   H^4 S A \log(T)^2\zeta\\
  &\leq 126  e^{127} \log(T) \sqrt{\zeta H^3 S A T}+ 3527 e^{127} H^4 S A  \log(T)^2\zeta\,.
\end{align*}
It remains to relate the optimistic regret with the regret. Thanks to Lemma~\ref{lem:optimism} we have
\begin{align*}
  \Vstar_1(s_1)-V_h^{\pi^{t+1}}(s_1) \leq \uV_1^t(s_1) - V_1^{\pi^{t+1}}(s_1)\,.
\end{align*}
This allows us to conclude
\begin{align*}
  R^T \leq \tR_1^T \leq  126  e^{127} \log(T) \sqrt{\zeta H^3 S A T}+ 3527 e^{127} H^4 S A  \log(T)^2\zeta\,.
\end{align*}
\end{proof}

\newpage
\section{Technical lemmas}
\label{app:technical}

\subsection{A Bellman-type equation for the variance}
\label{app:Bellman_variance}
We reproduce in this section the law of total variance from~\cite{menard2020fast}. For a deterministic policy $\pi$ we define Bellman-type equations for the variances as follows
\begin{align*}
  \Qvar_h^\pi(s,a) &\triangleq \Var_{p_h}{V_{h+1}^\pi}(s,a) + p_h \Vvar^\pi_{h+1}(s,a)\\
  \Vvar_h^\pi(s) &\triangleq \Qvar^\pi_h (s, \pi(s))\\
  \Vvar_{H+1}^\pi(s)&\triangleq0,
\end{align*}
where $\Var_{p_h}(f)(s,a) \triangleq \E_{s' \sim p_h(\cdot | s, a)} \big[(f(s')-p_h f(s,a))^2\big]$ denotes the \emph{variance operator.}
 In particular, the function $s \mapsto \Vvar_1^\pi(s)$ represents the average sum of the local variances $\Var_{p_h}{V_{h+1}^\pi}(s,a)$ over a trajectory following the policy $\pi$, starting from $(s, a)$. Indeed, the definition above implies that
 \[\Vvar_1^\pi(s_1) = \sum_{h=1}^H\sum_{s,a} p_h^\pi(s,a) \Var_{p_h}(V_{h+1}^\pi)(s,a).
 \]
 The lemma below shows that we can relate the global variance of the cumulative reward over a trajectory to the average sum of local variances.
\begin{lemma}[Law of total variance]  For any deterministic policy $\pi$ and for all $h\in[H]$,
  \[
  \E_\pi\!\left[  \left(\sum_{h'=h}^H r_{h'}( s_{h'},a_{h'}) - Q_h^\pi(s_h,a_h)\right)^{\!\!2}\middle| (s_h,a_h)=(s,a) \right] = \Qvar_h^\pi(s,a).
  \]
In particular,
\[
\E_\pi\!\left[ \left(\sum_{h=1}^H r_{h}( s_{h},a_{h}) - V_1^\pi(s_1)\right)^{\!\!2} \right] = \Vvar_1^\pi(s_1) = \sum_{h=1}^H\sum_{s,a} p_h^\pi(s,a) \Var_{p_h}(V_{h+1}^\pi)(s,a).
\]
\label{lem:law_of_total_variance}
\end{lemma}

\subsection{Weights and counts}
\label{app:count}

\begin{lemma}
	\label{lem:sum_1_over_n}
	 For $T\in\N^\star$ and $(u_t)_{t\in\N^\star},$ for a sequence where  $u_t\in[0,1]$ and $U_t \triangleq \sum_{l=1}^t u_\ell$, we get
	\[
		\sum_{t=0}^T \frac{u_{t+1}}{U_t\vee 1} \leq 4\log(U_{T+1}+1).
	\]
In particular if $T+1\geq 2$,
\[
\sum_{t=0}^T \frac{u_{t+1}}{U_t\vee 1} \leq 8\log(T+1)\,.
\]
\end{lemma}
\begin{proof}
	Notice that	\begin{align*}
		\sum_{t=0}^T \frac{u_{t+1}}{U_t\vee 1} &\leq 4 \sum_{t=0}^T \frac{u_{t+1} }{2U_t + 2} \\
		&\leq  4\sum_{t=0}^T \frac{U_{t+1}-U_{t}}{U_{t+1} + 1}\\
		&\leq 4\sum_{t=0}^T \int_{U_t}^{U_{t+1}} \frac{1}{x+1} \mathrm{d}x\\
		& = 4\log(U_{T+1}+1).
	\end{align*}
\end{proof}

\begin{lemma}
  \label{lem:properties_weights}
  For all $(s,a)\in\cS\times\cA$ it holds
  \begin{align*}
    \sum_{k=l}^t \chi_h^{k+1}(s,a) \teta_h^{k,l}(s,a) &\leq \left(1+\frac{1}{H}\right)\chi_h^l(s,a)\,,\\
      \sum_{k=1}^t \teta_h^{t,k}(s,a) &= 1 \qquad \text{ if } n_h^t(s,a)>0\,.
  \end{align*}
\end{lemma}
\begin{proof}
Note that if $\chi_h^l(s,a)=0$ then $\teta_h^{k,l}(s,a) = 0$ for all $k\geq l$ and the first inequality is true. Now assume that $\chi_h^l(s,a)>0$ and thus $n_h^t(s,a)\geq 1$. For $n,m\geq 1$ defined
\[
\teta^{n,m} = \frac{H+1}{H+m} \prod_{j=m+1}^n \frac{n-1}{H+n}\,,
\]
remark that
\[
\sum_{k=l}^t \chi_h^{k+1}(s,a) \teta_h^{k,l}(s,a)  \leq \sum_{n = n_h^l(s,a)}^{n_h^t(s,a)} \teta^{n,n_h^l(s,a)}\,.
\]
We will prove by induction that for all $N\geq m\geq 1$, which will implies the inequality we want to prove, that
\[
1+\frac{1}{H} -\sum_{n=m}^N \teta^{n,m} = \teta^{N,m} \frac{N}{H}\,.
\]
For $N=m$ we have
\[
1+\frac{1}{H} -\teta^{m,m} = \frac{H+1}{H}\frac{m}{H+m} = \teta^{m,m} \frac{m}{H}\,.
\]
Then if we assume that the result is true for $N$ then we obtain
\[
1+\frac{1}{H} - \sum_{n=m}^{N+1} \teta^{n,m} = \teta^{N+1,m}\left(\frac{H+N+1}{H}-1\right) = \teta^{N+1,m} \frac{N+1}{H}\,.
\]
The equality can be proved by induction using that for all $t$
\[
\sum_{k=1}^t \teta_h^{t,k}(s,a)= \eta_h^t(s,a)+ \big(1-\eta_h^t(s,a)\big)\sum_{k=1}^{t-1} \teta_h^{t-1,k}(s,a)\,,
\]
and there exists $k\leq t$ such that  $\teta_h^{t,k}(s,a) = 1$ because $n_h^t(s,a)>0$.
\end{proof}

\begin{lemma}
  \label{lem:sum_1_over_n_history}
  For all $(s,a)\in\cS\times\cA$ and $t\leq T-1$ (with $T\geq 2$), it holds
  \begin{align*}
    \chi_h^{l}(s,a)\sum_{k=l}^t  \frac{\chi_h^{k+1}(s,a)}{\tn_h^k(s,a)} &\leq 8\log(T)\chi_h^l(s,a)\,.
  \end{align*}
\end{lemma}
\begin{proof}
If $\chi_h^{l}(s,a)=0$ then the inequality is trivially true. Else $\chi_h^{l}(s,a)>0$ and using Lemma~\ref{lem:sum_1_over_n} we get
\begin{align*}
  \sum_{k=l}^t  \frac{\chi_h^{k+1}(s,a)}{\tn_h^k(s,a)} &\leq \sum_{k=0}^{T-1}  \frac{\chi_h^{k+1}(s,a)}{\tn_h^k(s,a)}\leq 8\log(T)\,.
\end{align*}
\end{proof}

\subsection{Inequality for the variance}
\label{app:variance}

\begin{lemma}
	\label{lem:switch_variance}
	For $p,q\in\Sigma_S$, for $f,g:\cS\mapsto [0,b]$ two functions defined on $\cS$, we have that
	\begin{align*}
 \Var_p(f) &\leq 2 \Var_p(g) +2 b p|f-g|\quad\text{and} \\
 \Var_p(f^2) &\leq 2b\Var_p(f),
\end{align*}
where we denote the absolute operator by $|f|(s)= |f(s)|$ for all $s\in\cS$.
\end{lemma}
\begin{proof}
First note that
\[
\Var_p(f) = p(f-g+ g-p g + p g- p f)^2 \leq 2 p(f-g - p f + p g)^2 +2 p(g-p g)^2 = 2\Var_p(f-g)+2\Var_p(g).
\]
From the above we can immediately conclude the proof of the first inequality with
\[
\Var_p(f-g) \leq p(f-g)^2 \leq b p|f-g|,
\]
where we used that for all $s\in\cS$, $0\leq |f(s)-g(s)| \leq b$. For the second inequality let $x\sim p$ be independent of $y\sim p$, then we have
\begin{align*}
  \Var_p (f^2) &= \frac{1}{2}\E_{x\sim p, y\sim p}\left[ \big(f(x)^2-f(y)^2\big)^2 \right]\\
  &= \frac{1}{2}\E_{x\sim p, y\sim p}\left[ \big(f(x)+f(y)\big)^2  \big(f(x)-f(y)\big)^2 \right]\\
  &\leq 2b^2 \Var_p(f)\,.
\end{align*}
\end{proof}

\end{document}